\documentclass[12pt]{article}
\usepackage[letterpaper,top=1in,bottom=1in,left=1in,right=1in]{geometry}
\usepackage[bottom]{footmisc}
\usepackage{enumitem}

\usepackage{amssymb,amsmath,amsthm}
\usepackage{mathtools}
\mathtoolsset{showonlyrefs}

\newtheoremstyle{mystyle}{5pt}{5pt}{\itshape}{}{\bfseries}{.}{5pt}{} \theoremstyle{mystyle} 
\theoremstyle{mystyle} \newtheorem{cor}{Corollary}
\theoremstyle{mystyle} \newtheorem{lem}{Lemma}
\theoremstyle{mystyle} \newtheorem{prop}{Proposition}
\theoremstyle{mystyle} \newtheorem{thm}{Theorem}
\theoremstyle{mystyle} \newtheorem{ass}{Assumption}
\theoremstyle{mystyle} \newtheorem{defn}{Definition}
\theoremstyle{mystyle} \newtheorem{rem}{Remark}

\newcommand{\E}{\mathbb{E}} 
\renewcommand{\P}{\mathbb{P}} 
\newcommand{\R}{\mathbb{R}} 
\newcommand{\N}{\mathbb{N}} 
\newcommand{\Z}{\mathbb{Z}}

\renewcommand{\epsilon}{\varepsilon}
\DeclareMathOperator*{\argmax}{arg\,max}

\DeclarePairedDelimiter{\ceil}{\lceil}{\rceil}
\DeclarePairedDelimiter{\floor}{\lfloor}{\rfloor}

\usepackage[linesnumbered,ruled]{algorithm2e}
\makeatletter
\renewcommand{\@algocf@capt@plain}{above}% formerly {bottom}
\makeatother

\usepackage{hyperref}
\hypersetup{
    colorlinks=true,
    linkcolor=blue,
   citecolor=blue
}

\usepackage{color}

\begin{document}

\title{One-bit feedback is sufficient for upper confidence bound policies} 
\author{Daniel Vial\footnote{University of Texas at Austin, dvial@utexas.edu}, Sanjay Shakkottai\footnote{University of Texas at Austin,
sanjay.shakkottai@utexas.edu}, R.\ Srikant\footnote{University of Illinois at Urbana–Champaign,
rsrikant@illinois.edu}}
\maketitle

\begin{abstract}
    We consider a variant of the traditional multi-armed bandit problem in which each arm is only able to provide one-bit feedback during each pull based on its past history of rewards. Our main result is the following: given an upper confidence bound policy which uses full-reward feedback, there exists a coding scheme for generating one-bit feedback, and a corresponding decoding scheme and arm selection policy, such that the ratio of the regret achieved by our policy and the regret of the full-reward feedback policy asymptotically approaches one. 
\end{abstract}

\section{Introduction} \label{secIntro}

The multi-armed bandit (MAB) problem is a classical model for online learning. Among its many applications, those arising in wireless sensor networks and data centers are particularly relevant to this work (see, e.g., \cite{avner2016multi,li2016collaborative,semasinghe2017game,tekin2014distributed}). Such applications feature cooperative agents, e.g., base stations/remote sensors in a sensor network or servers in a data center. In certain situations, the agents making decisions and those observing and/or storing the effects of decisions are separate entities. For example, a base station may control a remote sensor that makes observations, or data relevant to one server's algorithm may be stored on another. To model such situations in the MAB framework, where the decision maker perfectly observes and recalls rewards (or at least observes new rewards and recalls the mean of past rewards), one must assume the observing agents can transmit rewards to the decision maker. If rewards are continuous-valued (or even if the alphabet is finite but large), this may be infeasible in applications like sensor networks, where bandwidth is severely constrained.

Motivated by this issue, we consider a novel variant of the MAB problem in which the agent making decisions and those observing rewards can only communicate in a bandwidth-limited fashion. More specifically, we consider a set of $K+1$ agents cooperating to minimize regret from a $K$-armed stochastic bandit. The agents are arranged in a leader-follower configuration, with each of the $K$ followers located at a distinct arm. At each time $t$, the leader broadcasts an arm $A_t$, and the follower located at $A_t$ pulls this arm, observes a real-valued reward drawn from an unknown distribution $\nu_{A_t}$, and transmits a $B$-bit \textit{message} (which may only depend on past rewards generated by $A_t$) to the leader. No other communication is permitted; in particular, followers at arms $k \neq A_t$ do not transmit messages at time $t$.

In this setting, it is natural to assume bounded rewards, since messages will inevitably involve quantization. To simplify the discussion, we further restrict to $[0,1]$-valued rewards, as well as $B=1$ (though our ideas can be extended to bounded rewards and $B \in \N$). Under these assumptions, the leader must decide which arm to pull using one-bit feedback instead of the actual, continuous-valued rewards, distinguishing our setting from the standard MAB framework. Our goal is to design a \textit{message policy} $\pi^m$, which dictates how followers map their arms' reward histories to one-bit messages, and an \textit{arm policy} $\pi^a$, which dictates how the leader maps the message history to $A_t$, to minimize the expected cumulative \textit{regret}
\begin{equation}\label{eqRegretIntro}
R_n ( \pi^a , \pi^m ) =   \sum_{t=1}^n \E_{\pi^a, \pi^m} \left[ \max_{k \in [K]} \mu_k - \mu_{A_t}  \right] .
\end{equation}
Here $n \in \N$ is the horizon, $[K] = \{1,\ldots,K\}$ is the set of arms, and $\mu_k$ is the expected reward from arm $k$. Thus, $R_n ( \pi^a , \pi^m )$ is the difference in expected cumulative rewards when the best arm is known \textit{a priori} (in which case the leader always broadcasts $\argmax_{ k \in [K] } \mu_k$) and when following the policies $\pi^a, \pi^m$. See Section \ref{secSetting} for a formal problem definition.

\subsection{Our contributions}

The main contributions of this work are the design and analysis of message and arm policies for which the regret \eqref{eqRegretIntro} is essentially equivalent to the regret one would suffer in the standard MAB setting. Before making this precise, we describe the proposed policies.

For the message policy (see Section \ref{secMessage}), we use a simple idea in which followers group several one-bit messages into multi-bit \textit{packets} that encode quantized estimates of their arms' average rewards. The key trade-off is packet length. On the one hand, quantization error decreases as packet length grows. On the other hand, the leader must wait for the entire packet to arrive to exploit this increased accuracy; since the follower can only transmit one bit per arm pull, long packets require many pulls of arms which may yield low rewards. Our message policy balances this trade-off by dynamically increasing packet length with the number of pulls, in a manner that (approximately) minimizes regret (see Remark \ref{remPacketLength}).

For the arm policy (see Section \ref{secArm}), we take inspiration from \textit{upper confidence bound} (UCB) policies \cite{agrawal1995sample} in the MAB setting. In particular, we consider policies of the form
\begin{equation}\label{eqGeneralUcb}
\tilde{A}_t = \argmax_{ k \in [K] } g_t ( \hat{\mu}_k ( t-1 ) , \tilde{T}_k(t-1) ) ,
\end{equation}
where $\hat{\mu}_k ( t-1 )$ and $\tilde{T}_k(t-1)$ are the average reward from arm $k$ and the number of pulls of $k$ before time $t$, and $g_t : [0,1] \times \Z_+ \rightarrow \R_+$ is some function\footnote{Here and moving forward, $\Z_+ = \{ 0,1,\ldots\}$ and $\R_+ = [0,\infty)$.}. In Section \ref{secUcb}, we define a class of such functions, which includes those used in the \texttt{UCB1} \cite{auer2002finite} and \texttt{KL-UCB} \cite{cappe2013kullback,garivier2011kl,maillard2011finite} MAB policies. From this class, we define a corresponding class of arm policies for the leader-follower setting, in which (roughly) the leader replaces the average reward in \eqref{eqGeneralUcb} by the quantized estimate most recently transmitted by the follower. We write $R_n( \pi^a(g),\pi^m)$ for the regret \eqref{eqRegretIntro} when using our message policy and the arm policy defined by $g = \{ g_t \}_{t \in \N}$.

Our main result shows that for this class of UCB policies, performance degradation compared to the standard MAB setting is negligible. More specifically, we compare the regret $R_n ( \pi^a(g) , \pi^m )$ to the regret incurred when using $g$ with full reward observations, i.e.,
\begin{equation}\label{eqRegretIntroMAB}
\tilde{R}_n ( g ) =   \sum_{t=1}^n \E_g \left[ \max_{k \in [K]} \mu_k - \mu_{\tilde{A}_t} \right] , 
\end{equation}
where $\{ \tilde{A}_t \}_{t=1}^n$ are chosen according to \eqref{eqGeneralUcb}. We then prove (see Theorem \ref{thmMain})
\begin{equation} \label{eqMainResultIntro}
\limsup_{n \rightarrow \infty} R_n ( \pi^a(g) , \pi^m ) / \tilde{R}_n ( g) \leq 1 .
\end{equation}
Thus, our message policy ensures no asymptotic performance loss for UCB-based arm policies when the leader observes one-bit messages instead of $[0,1]$-valued rewards. In other words, one-bit feedback is sufficient for UCB policies. In addition to \eqref{eqRegretIntroMAB}, we derive finite-horizon regret bounds for the \texttt{UCB1} and \texttt{KL-UCB} variants of our algorithm; see Corollaries \ref{corFiniteUCB1} and \ref{corFiniteKL}. Finally, we demonstrate the performance of our policies numerically in Section \ref{secExp}.

\subsection{Related work}

The setting studied in this paper is novel, to the best of our knowledge. However, it is not the first variant of the MAB problem in which rewards are imperfectly observed. For example, several papers have considered the case where rewards are corrupted by an adversary before being observed by the decision maker (see, e.g., \cite{gupta2019better,lykouris2018stochastic}); this is clearly distinct from our cooperative setting. 

As discussed above, our policies adapt algorithms like \texttt{UCB1} \cite{auer2002finite} and \texttt{KL-UCB} \cite{cappe2013kullback,garivier2011kl,maillard2011finite} to the leader-follower setting. We define these algorithms formally in Section \ref{secPrelim}. Like our policies, they apply to $[0,1]$-valued rewards (and general bounded rewards, after normalization). We note that, in the MAB setting, \texttt{KL-UCB} attains asymptotically optimal regret for Bernoulli rewards and lower regret than \texttt{UCB1} for general $[0,1]$-valued rewards. In this work, we draw similar conclusions for the corresponding leader-follower algorithms, i.e., the \texttt{KL-UCB} version of our policy outperforms the \texttt{UCB1} version. However, we include \texttt{UCB1} in our discussion to demonstrate the generality of the class of UCB policies covered by our analysis.

\section{Preliminaries} \label{secPrelim}

\subsection{Formal definition of the setting} \label{secSetting}

Let $K \in \{2,3,\ldots\}$ denote the number of arms, and let $\nu_1, \ldots , \nu_K$ be distributions over $[0,1]$. For each $k \in [K] = \{1,\ldots,K\}$, let $\{ X_{k,s} \}_{s \in \N}$ be i.i.d.\ random variables distributed as $\nu_k$. We then define a sequential game as follows: at each time $t \in \N$,
\begin{itemize}
\item The leader broadcasts an arm $A_t = \pi_t^a ( \{   M_{k,s} \}_{ k \in [K] , s \in [T_k(t-1)]  } )$, where $\pi_t^a$ is a $[K]$-valued function, $T_k(t-1) = | \{ t' \in [t-1] : A_{t'} = k \} |$ is the number of broadcasts of $k$ before time $t$, and $M_{k,s}$ is the message received when $k$ was broadcast for the $s$-th time.
\item Follower $A_t$ pulls its arm, observes the reward $X_{A_t, T_{A_t}(t) }$, and transmits the message $M_{A_t, T_{A_t}(t) } = \pi_t^m ( \{ X_{A_t,s} \}_{s \in [ T_{A_t}(t) ] } )$ to the leader, where $\pi_t^m$ is a binary-valued function.
\end{itemize}
Here $\pi^a = \{ \pi^a_t \}_{t \in \N}$ and $\pi^m = \{ \pi^m_t \}_{t \in \N}$ are the arm and message policies discussed in Section \ref{secIntro}. Note $\pi_t^a$ takes as input the messages $\{   M_{k,s} \}_{ k \in [K] , s \in [T_k(t-1)]  }$ that the leader has received before time $t$, while $\pi_t^m$ takes as input the rewards $\{ X_{A_t,s} \}_{s \in [ T_{A_t}(t) ] }$ follower $A_t$ has observed up to and including $t$. We reiterate that followers $k \neq A_t$ do not transmit at time $t$. 

As in Section \ref{secIntro}, we let $\mu_k = \E X_{k,1}$ denote the expected reward from arm $k$. For simplicity, we assume arms are sorted by their expected rewards, there is a unique best arm, and none of the arm distributions are degenerate: $0 < \mu_K \leq \cdots \leq \mu_2 < \mu_1 < 1$. For $k \in \{2,\ldots,K\}$, we define the $k$-th arm gap $\Delta_k = \mu_1 - \mu_k \in (0,1)$. We can then rewrite the regret \eqref{eqRegretIntro} as
\begin{equation}\label{eqRegret}
R_n ( \pi^a , \pi^m ) =   \sum_{t=1}^n \E_{\pi^a, \pi^m} [ \mu_1 - \mu_{A_t}  ] = \sum_{k=2}^K \Delta_k \E_{\pi^a, \pi^m} [ T_k(n) ] ,
\end{equation}
where the subscript $\pi^a, \pi^m$ indicates that $\{ A_t \}_{t \in \N}$ are chosen according to $\pi^a, \pi^m$\footnote{When clear from context, we discard this subscript.}.

To clarify notation, we also define the standard MAB setting and UCB policies discussed in Section \ref{secIntro}. Here the game is as follows. For each $t \in [K]$, the decision maker pulls $\tilde{A}_t = t$ and observes $X_{ t,1 }$. For each $t \in \{ K+1,K+2,\ldots \}$, the decision maker pulls
\begin{equation}\label{eqTildeAt}
\tilde{A}_t = \argmax_{ k \in [K] } g_t ( \hat{\mu}_{k , \tilde{T}_k(t-1) } , \tilde{T}_k(t-1) ) ,
\end{equation}
where $g_t : [0,1] \times \N \rightarrow \R_+$ is a function we call the \textit{decision function}, $\tilde{T}_k(t-1) = | \{ t' \in [t-1] : \tilde{A}_t = k \} |$ is the number of pulls of arm $k$ before time $t$, and $\hat{\mu}_{k , \tilde{T}_k(t-1)} = \sum_{s=1}^{\tilde{T}_k(t-1)} X_{k,s}$ is the empirical mean of rewards generated by arm $k$ before time $t$. The decision maker then observes the reward $X_{\tilde{A}_t, \tilde{T}_{\tilde{A}_t}(t)}$ generated by arm $\tilde{A}_t$.

\begin{rem} \label{remDecisionFunction}
Note we assume arms are played in round robin once (i.e., $\tilde{A}_t = t$ for $t \leq K$), and then chosen by the decision function $g_t$, which only depends on the empirical mean and number of pulls of $k$. For instance, \texttt{UCB-1} sets $g_t(x,s) = x + \sqrt{2 \log(t) / s}$; in words, the decision function is the empirical mean plus a bonus that decreases in the number of pulls. The idea is to balance \emph{exploitation} (i.e., pulling arms with large empirical means) and \emph{exploration} (i.e., pulling all arms often enough that the empirical means are accurate). 
\end{rem}

Finally, we note that, similar to \eqref{eqRegret}, we can rewrite the regret \eqref{eqRegretIntroMAB} as
\begin{equation}\label{eqRegretMAB}
\tilde{R}_n(g) =  \sum_{t=1}^n \E_g [ \mu_1 - \mu_{ \tilde{A}_t }  ] = \sum_{k=2}^K \Delta_k \E_g [ \tilde{T}_k(n) ] .
\end{equation}

\subsection{Upper confidence bound (UCB) policies} \label{secUcb}

We next state assumptions on the decision function that will allow us to make a meaningful comparison between the MAB regret \eqref{eqRegretMAB} and the regret of our forthcoming policy. First, we have an intuitive monotonicity condition, which says the decision function should increase in time and average reward, while decreasing in the number of plays (see Remark \ref{remDecisionFunction}).
\begin{ass} \label{assMono}
$g_t(x,s)$ is increasing in $t$, increasing in $x$, and decreasing in $s$.
\end{ass}

Next, we require the decision function to converge to the empirical mean as the number of pulls grows (i.e., as the empirical mean converges to the true mean with high probability).
\begin{ass} \label{assConvSimple}
For any $x \in (0,1)$ and any $t \in \N$, $\lim_{s \rightarrow \infty} g_t(x,s) = x$.
\end{ass}

The preceding assumptions do not preclude the case where the decision function equals the empirical mean, i.e., $g_t(x,s) = x$. To ensure the policy is sufficiently explorative, we also require $g_t(x,s)$ to be ``sufficiently above'' $x$. The precise condition, motivated by the analysis of \texttt{KL-UCB}, involves two functions. First, define $f : \N \rightarrow \R_+$ by
\begin{equation}
f(t) = 1 + t ( \log t)^2\ \forall\ t \in \N .
\end{equation}
Next, define the (Bernoulli) \textit{Kullback–Leibler divergence} $d : [0,1]^2 \rightarrow \R_+ \cup \{\infty\}$ by
\begin{equation} \label{eqKLdefn}
d(p,q) = p \log \frac{p}{q} + (1-p) \log \frac{1-p}{1-q} ,
\end{equation}
where $0 \log 0 = 0 \log \frac{0}{0} = 0, r \log \frac{r}{0} = \infty\ \forall\ r > 0$ by convention. We then have the following.
\begin{ass} \label{assUpper}
For any $x \in (0,1)$, $t \in \N$, and $s \in \N$, we have $g_t(x,s) > x$. Furthermore, if $g_t(x,s) < 1$, then $d( x , g_t(x,s) ) \geq \log ( f(t) ) / s$.
\end{ass}

Assumptions \ref{assMono}-\ref{assUpper} will allow us to derive finite-time regret bounds for our policies. To prove \eqref{eqMainResultIntro}, we require further assumptions regarding the asymptotic behavior of the decision function. These will be stated in terms of the function $S_t^g(x_1, x_2)$, defined as follows.
\begin{defn} \label{defnConv}
For any $t \in \N$ and any $x_1 , x_2 \in (0,1)$ s.t.\ $x_1 < x_2$, let $S_t^g(x_1,x_2) = \min \{ s \in \N : g_t(x_1,s) \leq x_2 \}$, where $S_t^g(x_1,x_2) = \infty$ by convention if $g_t(x_1,s) > x_2\ \forall\ s \in \N$.
\end{defn}
In words, $S_t^g(x_1,x_2)$ is the number of pulls $s$ for the decision function $g_t(x_1,s)$ to fall below $x_2$. When clear from context, we discard the superscript $g$ and simply write $S_t(x_1,x_2)$. Note the corner case $g_t(x_1,s) > x_2\ \forall\ s \in \N$ cannot occur when Assumption \ref{assConvSimple} holds.

The next assumption states $S_t^g(x_1, x_2) = o(t)$. This is trivially satisfied by any reasonable UCB policy, for which $S_t^g(x_1,x_2) = O(\log t)$ (see Proposition \ref{propAssumptions} for examples).
\begin{ass} \label{assConv1}
For any $x_1 , x_2 \in (0,1)$ s.t.\ $x_1 < x_2$, $\lim_{t \rightarrow \infty} S_t^g(x_1,x_2) / t = 0$.
\end{ass}

Our final assumption is a continuity condition, which essentially says that if $\delta$ is small and $t$ is large, then $S_t ( x_1 + \delta , x_2 - \delta ) \approx S_{ (1-\delta) t } ( x_1 - \delta , x_2 + \delta )$. The exact condition is stated in terms of vanishing sequences $\{ \delta_i \}_{i \in \N}$ satisfying
\begin{equation}
0 < x_1 + \delta_i < x_2 - \delta_i < 1 , \quad  0 < x_1 - \delta_i < x_2 + \delta_i < 1 \quad \forall\ i .
\end{equation}
These inequalities ensure that, when Assumption \ref{assConvSimple} holds, the quantities $S_t^g ( x_1 + \delta_i  , x_2 - \delta_i )$ and $S_{ \ceil{(1-\delta_i)t} }^g ( x_1 - \delta_i , x_2 + \delta_i )$ appearing in Assumption \ref{assConv2} are well-defined and finite.
\begin{ass} \label{assConv2}
For any $x_1 , x_2 \in (0,1)$ s.t.\ $x_1 < x_2$, and any $\{ \delta_i \}_{i \in \N} \subset ( 0 , \min \{ x_1 , 1 - x_2 , \frac{x_2-x_1}{2} \} )$ s.t.\ $\lim_{i \rightarrow \infty} \delta_i = 0$,
\begin{equation}
\lim_{i \rightarrow \infty}  \lim_{t \rightarrow \infty} \frac{ S_t^g ( x_1 + \delta_i  , x_2 - \delta_i ) }{ S_{ \ceil{(1-\delta_i)t}  }^g ( x_1 - \delta_i , x_2 + \delta_i ) } = 1.
\end{equation}
\end{ass}

Recall our goal is to understand the regret behavior of popular bandit policies like \texttt{UCB1} and \texttt{KL-UCB}. These policies correspond to the decision functions
\begin{equation}
g^{\texttt{UCB1}}_t ( x , s ) = x + \sqrt{ 2 \log (t) / s } , \quad  g^{\texttt{KL-UCB}}_t ( x , s )  = \max \left\{ y \in [0,1] : d(x,y) \leq \log ( f(t)) / s \right\} .
\end{equation}
We refer the reader to \cite[Chapters 7-10]{lattimore2020bandit} for a detailed discussion of these algorithms. In both cases, the decision function belongs to the class defined above, as formalized by the following proposition.

\begin{prop} \label{propAssumptions}
$g^{\texttt{UCB1}}$ and $g^{\texttt{KL-UCB}}$ satisfy Assumptions \ref{assMono}-\ref{assConv2}, and for any $t \in \N$ and any $x_1, x_2 \in (0,1)$ s.t.\ $x_1 < x_2$,
\begin{equation}
S_t^{ g^{\texttt{UCB1}} } (x_1,x_2) = \ceil*{ 2 \log (t) / ( x_2 - x_1 )^2  } ,  \quad S_t^{ g^{\texttt{KL-UCB}} } (x_1,x_2) = \ceil*{ \log (f(t)) / d ( x_1 , x_2 )  } .
\end{equation}
\end{prop}
\begin{proof}
The proof is straightforward and can be found in Appendix \ref{secProofAssumptions} in the supplementary material.
\end{proof}

\section{Algorithm}

\subsection{Message policy} \label{secMessage}

Our message policy is notationally cumbersome, so we begin with an informal description in Figure \ref{figMessage}. Here we depict the policy from the perspective of follower $k \in [K]$. The axis shows time and the stars represent times $t$ at which arm $A_t = k$ is pulled (we skew the axis so these stars are evenly spaced). The middle of the diagram shows how the pulls at top are mapped to the messages $M_{k,s}$ at bottom. The interpretation is as follows:
\begin{itemize}
    \item After its first arm pull, follower $k$ observes reward $X_{k,1}$ and computes a one-bit uniform quantization of the empirical mean $\hat{\mu}_{k,1} = X_{k,1}$ (we formally define this quantization shortly). This bit is then transmitted as the first message $M_{k,1}$.
    \item After its second pull, follower $k$ observes $X_{k,2}$ and quantizes $\hat{\mu}_{k,2} = \frac{ X_{k,1} + X_{k,2} }{2}$ using \textit{two bits}. The first of these bits is transmitted immediately as the message $M_{k,2}$; the second bit is transmitted as $M_{k,3}$ in the time slot of the third pull. 
    Moving forward, we refer to $(M_{k,2} , M_{k,3})$ (and analogous sequences) as \textit{packets}.
    \item Similarly, the next packet $( M_{k,4}, M_{k,5} )$ is a two-bit quantization of $\hat{\mu}_{k,4} = \frac{ \sum_{s=1}^4 X_{k,s} }{4}$. The next four packets start at the time slots of pulls $6$, $9$, $12$, and $15$; each contains a \textit{three-bit} quantization of the corresponding empirical mean. Thereafter, packet length increases to four bits. In general, packet length grows as follows: $j$-bit packets are transmitted $2^{j-1}$ times, and then packet length increases to $j+1$ bits.
\end{itemize}

\begin{figure}
\centering
\includegraphics[width=\textwidth]{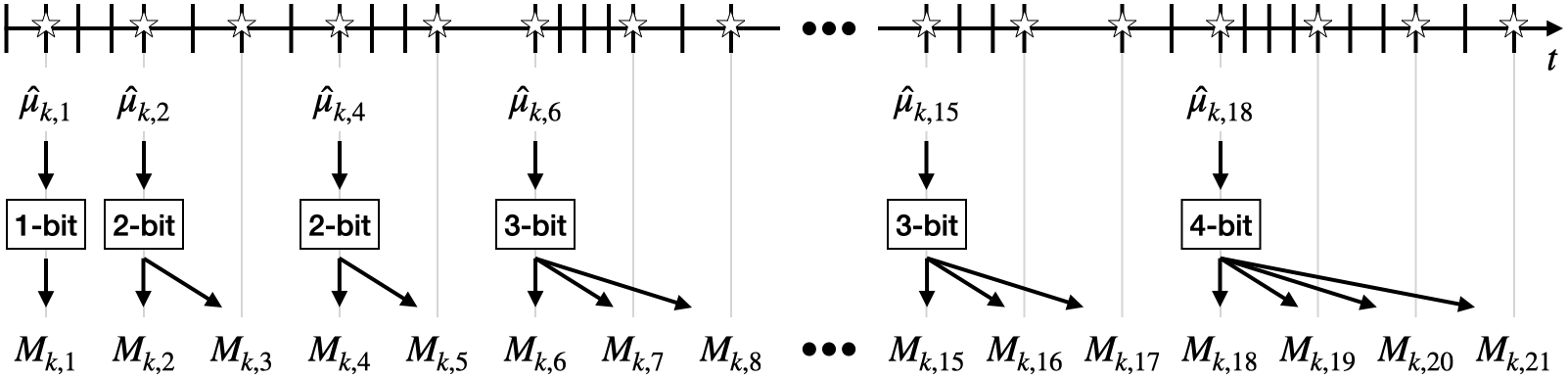}
\caption{Depiction of message policy from perspective of follower $k$.} \label{figMessage}
\end{figure}

As discussed in Section \ref{secIntro}, the key feature of this policy is that packet length increases, and thus quantization error decreases, as the number of pulls grows. The precise packet length schedule shown in Figure \ref{figMessage} ensures that the $i$-th packet contains $\Theta(\log i)$ bits, which appears to be order-wise optimal in terms of regret (see Remark \ref{remPacketLength}).

\subsubsection{Formal definition}

We now formalize this policy. We begin by defining $\tau, \iota : \Z_+ \rightarrow \Z_+$ by
\begin{gather}
\tau(i)  = 1 + (i+1) \ceil{ \log_2(i+1) } - 2^{ \ceil{ \log_2 (i+1) } }\ \forall\ i \in \Z_+ , \label{eqDefnTau} \\
\iota ( s ) = \max \{ i \in \Z_+ : \tau(i) \leq s \}\ \forall\ s \in \Z_+ . \label{eqDefnIota}
\end{gather}
Here $\tau(i)$ represents the \textit{number of pulls at which the $i$-th packet completes transmission} and defines the packet length schedule discussed above. Accordingly, $\iota(s)$ represents \textit{the number of packets transmitted after $s$ pulls}. Next, for $\alpha \in \Z_+$ we define
\begin{equation}\label{eqQuantDefn}
\Gamma_{\alpha} ( x ) = \ceil{ 2^{\alpha} x } - 1\ \forall\ x \in (0,1] , \quad \Gamma_{\alpha}(0) = 0 .
\end{equation}
Finally, we define a binary variable, i.e., the message $M_{k,s}$ transmitted by follower $k$ after its $s$-th arm pull by
\begin{equation}\label{eqMessageDefn}
M_{k,s} = \Gamma_{ s - \tau(\iota(s-1))  } ( \hat{\mu}_{k, 1 +\tau(\iota(s-1)  ) } )  - 2 \Gamma_{ s - \tau(\iota(s-1)) - 1 } ( \hat{\mu}_{k, 1 +\tau(\iota(s-1)  ) }   ) .
\end{equation}

\begin{rem}
Note $1 + \tau(\iota(s-1)) \leq s$ by \eqref{eqDefnIota}, so \eqref{eqMessageDefn} is feasible, in the sense that it is transmitted after the $s$-th arm pull and only depends on $\{ X_{k,s'} \}_{s'=1}^{ 1 + \tau(\iota(s-1)) } \subset \{ X_{k,s'} \}_{s'=1}^s$.
\end{rem}

\begin{rem}
The form of \eqref{eqMessageDefn} is motivated by the following identity: for $\alpha \in \N , x \in [0,1]$,
\begin{equation}\label{eqDecoderInformal}
\sum_{s=1}^{\alpha} 2^{\alpha-s} ( \Gamma_s (x) - 2 \Gamma_{s-1}(x) ) = \Gamma_{\alpha} ( x ) - 2^{\alpha} \Gamma_0(x) = \Gamma_{\alpha}(x) ,
\end{equation}
where we evaluated a telescoping sum and used $\Gamma_0(x) = 0$ by definition. In other words, using bits of the form $\Gamma_s(x) - 2 \Gamma_{s-1}(x)$ (i.e., the form of \eqref{eqMessageDefn}),
one can recover $\Gamma_{\alpha}(x)$, from which  the following quantized estimate of $x$ can be derived:
\begin{equation}\label{eqDecoderInformal2}
2^{-\alpha} ( \Gamma_{\alpha}(x) + 1 ) = 2^{-\alpha} \ceil{ 2^{\alpha} x }  \in [ x , x + 2^{-\alpha} ] ,
\end{equation}
i.e., we round $x$ up to the nearest element of $\{ 2^{-\alpha} i \}_{i \in [2^{\alpha}] }$ (see Remark \ref{remRoundUp}). The additional notation in \eqref{eqMessageDefn} maps the current pull $s$ to the appropriate empirical mean $\hat{\mu}_{k,1+\tau(\iota(s-1))}$ and to the proper bit of the encoding $s-\tau(\iota(s-1))$. %; see Figure \ref{figIotaTau1}.
\end{rem}

% \begin{figure}
% \centering
% \includegraphics[height=1.3in]{message}
% \caption{$\tau(\iota(s-1))+1$ maps the number of pulls $s$ to the first pull of the current packet; shown are the cases $s = \tau(\iota(s))$ (left) and $s \neq \tau(\iota(s))$ (right). Here each tick represents an arm pull and the boxes represent packets (with text showing packet index).} \label{figIotaTau1}
% \end{figure}

The proposed message policy is summarized in Algorithm \ref{algMessage}.

{\linespread{1}
\begin{algorithm} \label{algMessage}

\caption{Message policy $\pi^m$ (at follower $k$)}

Set $T_k(0) = 0$

\For{$t \in \N$}{

\uIf{$A_t = k$}{

Set $T_k(t) = T_k(t-1)+1$, observe $X_{k, T_k(t) }$, transmit $M_{k,T_k(t)}$ (defined in \eqref{eqMessageDefn})

}
\Else{

Set $T_k(t) = T_k(t-1)$

}

}

\end{algorithm}}

\subsection{Arm policy} \label{secArm}

To define the arm policy, we first specify how the leader decodes the messages \eqref{eqMessageDefn}. First, for each $k \in [K], s \in \N$, we define
\begin{equation}\label{eqDefnMuBar}
\bar{\mu}_{k,s} = \sum_{ s' = 1 + \tau( \iota(s) - 1) }^{ \tau(\iota(s)) } 2^{ \tau( \iota(s) - 1) - s'  } M_{k,s'}   + 2^{ \tau(\iota(s)-1) - \tau(\iota(s)) } .
\end{equation}
Recall $\tau(\iota(s)) \leq s$ by \eqref{eqDefnIota}, so after receiving $s$ messages from follower $k$, the leader can compute \eqref{eqDefnMuBar}. In particular, \textit{the leader can compute $\bar{\mu}_{k,T_k(t-1)}$ before its $t$-th broadcast}. Analogous to \eqref{eqDecoderInformal} and \eqref{eqDecoderInformal2}, we can bound the quantization error in \eqref{eqDefnMuBar}.
\begin{prop} \label{propMessageAccuracy}
Let $s \in \N$, $\alpha(s) = \tau ( \iota(s) ) - \tau ( \iota(s) - 1 )$ and $\eta(s) = 1 + \tau ( \iota(s) - 1 )$. Then
\begin{equation}
\bar{\mu}_{k,s} = \begin{cases} 2^{-\alpha(s)} \ceil{ 2^{\alpha(s)} \hat{\mu}_{k,\eta(s)} } , & \hat{\mu}_{k,\eta(s)} \in (0,1] \\ 2^{-\alpha(s)}, & \hat{\mu}_{k,\eta(s)} = 0 \end{cases}  \in [ \hat{\mu}_{k,\eta(s)} , \hat{\mu}_{k,\eta(s)} + 2^{-\alpha(s)} ] .
\end{equation}
\end{prop}
\begin{proof}
The result follows from the logic of \eqref{eqDecoderInformal} and \eqref{eqDecoderInformal2}. See Appendix \ref{secProofMessageAccuracy} for details.
\end{proof}

% \begin{figure}
% \centering
% \includegraphics[height=1.3in]{arm}
% \caption{$\eta(s)$ and $\alpha(s)$, respectively, map the current number of pulls $s$ to the number of pulls at start of most recently transmitted packet and length of this packet, respectively.} \label{figIotaTau2}
% \end{figure}

In light this, we define our arm policy (Algorithm \ref{algArm}): at $t \in [K]$, the leader broadcasts $A_t = t$ (arms are played in round robin); at $t \in \{K+1,K+2,\ldots,\}$, the leader broadcasts
\begin{equation}\label{eqArmPolicy}
A_t = \argmax_{ k \in [K] } g_t ( \bar{\mu}_{k, T_k(t-1)} , \eta(T_k(t-1)) ) .
\end{equation}
Thus, the leader replaces the average reward with the estimate \eqref{eqDefnMuBar}, and the number of pulls by the number of samples $\eta(T_k(t-1))$ comprising this estimate.

We return to comment on the functions $\tau(\iota(\cdot))$, $\alpha(\cdot)$, and $\eta(\cdot)$ in Proposition \ref{propMessageAccuracy}. First, by the earlier interpretations of $\tau$ and $\iota$, $\tau(\iota(\cdot))$ represents \textit{the number of pulls at which a packet most recently completed transmission}. Consequently, $\alpha(\cdot)$ is \textit{the length of the most recent packet}, and $\eta(\cdot)$ is \textit{the number of samples for the quantized estimate contained in this packet}. By these interpretations, $\alpha(\cdot)$ and $\eta(\cdot)$ dictate quantization and sampling error, respectively. From the chosen packet length schedule, one can easily show $\alpha(s) = \Theta ( \log s )$ and $\eta(s) = s - \Theta ( \alpha(s) )$. More precisely, we have the following.
\begin{prop} \label{propAlphaPsiMain}
Let $u \geq 16$ and $s \geq u + 3 \log_2 u$. Then $\eta(s) \geq \frac{s}{2} \vee u$ and $\alpha(s) \geq \log_2 \sqrt{s}$.
\end{prop}
\begin{proof}
The proof is elementary but tedious, so we defer it to Appendix \ref{secProofAlphaPsi}.
\end{proof}

{\linespread{1}\begin{algorithm}

\caption{Arm policy $\pi^a(g)$ with decision functions $g = \{ g_t \}_{t \in \N}$ (at leader)} \label{algArm}

\For{$t \in \{1,\ldots,K\}$}{

Broadcast $A_t = t$, receive $M_{t,1}$ (defined in \eqref{eqMessageDefn})

}

Set $T_k(K) = 1\ \forall\ k \in [K]$

\For{$t \in \{K+1,K+2,\ldots\}$}{

Broadcast $A_t = \argmax_{k \in [K]} g_t ( \bar{\mu}_{k,T_k(t-1)} ,  \eta ( T_k(t-1) ) )$ (defined in \eqref{eqDefnMuBar})

Set $T_{A_t}(t) = T_{A_t}(t-1)+1$, receive $M_{A_t, T_{A_t}(t) }$ (defined in \eqref{eqMessageDefn})

Set $T_k(t) =T_k(t-1)\ \forall\ k \in [K] \setminus \{ A_t \}$

}

\end{algorithm}}

\section{Results}

We now turn to the analysis of our policy. The key idea is that the policy's logarithmically-increasing packet lengths ensure the quantization errors introduced by the one-bit feedback restriction decay at the same rate as the sampling errors of the local arm mean estimates (see Remark \ref{remPacketLength} for more details). Consequently, the leader’s arm mean estimates are just as accurate (in an order sense) as the decision maker’s estimates in the corresponding MAB problem. This will allow us to prove \eqref{eqMainResultIntro}, i.e., that our one-bit feedback policy asymptotically incurs the same regret as the corresponding MAB policy, in Section \ref{secResultsAsymptotic}. We will then derive finite-horizon bounds for the \texttt{UCB1} and \texttt{KL-UCB} variants of our policy in Section \ref{secResultsFinite}.

\subsection{Asymptotic bound for general UCB-based policies} \label{secResultsAsymptotic}

Toward proving \eqref{eqMainResultIntro}, we first bound the regret of our policy. Here the bound is in terms of the function from Definition \ref{defnConv}; as mentioned above, we will later (in Section \ref{secResultsAsymptotic}) derive explicit bounds for \texttt{UCB1} and \texttt{KL-UCB}.

\begin{lem} \label{lemUpperFinite}
Let $g = \{ g_t \}_{t \in \N}$ satisfy Assumptions \ref{assMono}-\ref{assUpper}, and let $\pi^m$ and $\pi^a(g)$ be the policies described in Algorithms \ref{algMessage} and \ref{algArm}, respectively. Then for any $n \in \N$, the regret \eqref{eqRegret} satisfies
\begin{equation}
R_n(\pi^a(g),\pi^m) \leq \sum_{k=2}^K \Delta_k \min_{ \delta \in ( 0 , \frac{\Delta_k}{2} ) }  \Big( S_n ( \mu_k + \delta , \mu_1 - \delta ) + 3 \log_2 \Big( S_n ( \mu_k + \delta , \mu_1 - \delta ) \vee \frac{4}{\delta^2} \Big) + \frac{10}{\delta^2} \Big) .
\end{equation}
\end{lem}
\begin{proof}[Proof sketch]
The full proof can be found in Appendix \ref{secProofUpperFinite}, where we prove a slightly stronger result: for any arm $k \geq 2$ and any $\delta \in ( 0 , \frac{\Delta_k}{2} )$,
\begin{equation}\label{eqNumPlaysMain}
\E [ T_k(n) ] \leq S_n ( \mu_k + \delta , \mu_1 - \delta ) \vee \frac{4}{\delta^2} + 3 \log_2 \left( S_n ( \mu_k + \delta , \mu_1 - \delta  ) \vee \frac{4}{\delta^2} \right) + \frac{6}{\delta^2} .
\end{equation}

The proof of \eqref{eqNumPlaysMain} is loosely based on the analysis of \texttt{KL-UCB} presented in \cite[Chapter 10]{lattimore2020bandit}. First, we set $u = S_n ( \mu_k + \delta , \mu_1 - \delta ) \vee 4 \delta^{-2}$ and assume arm $k$ has been pulled $s_k = u + 3 \log_2 u$ times. Then by Propositions \ref{propMessageAccuracy} and \ref{propAlphaPsiMain},
\begin{equation}\label{eqMuKEtaKMain}
\bar{\mu}_{k,s_k} - \hat{\mu}_{k,\eta(s_k)} \leq 2^{-\alpha(s_k)} \leq 1 / \sqrt{s_k} \leq \delta / 2 , \quad \eta(s_k) \geq u \geq S_n ( \mu_k + \delta , \mu_1 - \delta ) .
\end{equation}
Thus, assuming $\hat{\mu}_{k,\eta(s_k)} \leq \mu_k + \delta / 2$, we obtain $\bar{\mu}_{k,s_k} \leq \mu_k + \delta$. Combined with the inequality for $\eta(s_k)$ in \eqref{eqMuKEtaKMain}, and using Assumption \ref{assMono} and Definition \ref{defnConv}, we conclude
\begin{equation}\label{eqArmKkey}
g_t ( \bar{\mu}_{k,s_k} , \eta ( s_k ) ) \leq g_n ( \mu_k + \delta , S_n ( \mu_k + \delta , \mu_1 - \delta ) ) \leq \mu_1 - \delta .
\end{equation}
On the other hand, if arm $1$ is pulled $s_1$ times and $\hat{\mu}_{1,\eta(s_1)} > \mu_1 - \delta$, then $\bar{\mu}_{1,s_1} > \mu_1 - \delta$ by Proposition \ref{propMessageAccuracy}, so by Assumptions \ref{assMono} and \ref{assUpper},
\begin{equation}\label{eqArm1key}
g_t ( \bar{\mu}_{1,s_1} , \eta(s_1) ) \geq g_t ( \mu_1 - \delta , \eta(s_1) ) > \mu_1 - \delta .
\end{equation}

To summarize, we have argued the following: if arm $k$ has been pulled $u + 3 \log_2 u$ times, and if $\hat{\mu}_{k,\eta(s_k)} \leq \mu_k + \delta/2$ and $\hat{\mu}_{1,\eta(s_1)} > \mu_1 - \delta$, then arm $1$'s index must exceed arm $k$'s index (see \eqref{eqArmKkey} and \eqref{eqArm1key}), so no additional pulls of arm $k$ will occur. The first two terms in \eqref{eqNumPlaysMain} account for these $u + 3 \log_2 u$ pulls, and the final term accounts for the low probability failure of the event $\{ \hat{\mu}_{k,\eta(s_k)} \leq \mu_k + \delta/2 , \hat{\mu}_{1,\eta(s_1)} > \mu_1 - \delta \}$. More precisely, $\hat{\mu}_{k,\eta(s_k)} \approx \mu_k$ with high probability since we assumed a lower bound on $s_k$ above. In contrast, we made no such assumption on $s_1$, so $\hat{\mu}_{1,\eta(s_1)}$ may be far from $\mu_1$ for small $s_1$. This is why Assumption \ref{assUpper} is needed: it ensures that when $s_1$ is small, arm $1$'s index is large (even if $\hat{\mu}_{1,\eta(s_1)}$ is small).
\end{proof}

\begin{rem} \label{remRoundUp}
Before \eqref{eqArm1key}, we used the fact that the quantizer rounds up, i.e., $\bar{\mu}_{1,\eta(s_1)} \geq \hat{\mu}_{1,\eta(s_1)}$. If instead it rounded down, the analogue of  Proposition \ref{propMessageAccuracy} would imply $\bar{\mu}_{1,\eta(s_1)} \geq \hat{\mu}_{1,\eta(s_1)} - 2^{-\alpha(s_1)}$, which is too loose in the case of small $s_1$ discussed at the end of the proof sketch. Thus, the fact that our quantizer rounds up appears to be an artifact of our analysis.
\end{rem}

\begin{rem} \label{remPacketLength}
As shown in the proof sketch, our packet length schedule ensures the quantization error $|\bar{\mu}_{k,s_k} - \hat{\mu}_{k,\eta(s_k)}|$ is of order $\delta$, the same as the sampling error $|\hat{\mu}_{k,\eta(s_k)} - \mu_k|$. This means that packet length grows as slowly as possible without introducing additional errors (in an order sense) due to quantization. Slow growth of the packet length is desirable for the following reason: if instead packet length grew as, e.g., $\alpha(s_k) = poly(s_k)$, we would need $s_k = u + poly ( u )$ to obtain the second inequality in \eqref{eqMuKEtaKMain}; this would change the $\log (  S_n ( \mu_k + \delta , \mu_1 - \delta) )$ term in Lemma \ref{lemUpperFinite} to $poly ( S_n ( \mu_k + \delta , \mu_1 - \delta) )$, increasing regret in an order sense. 
\end{rem}

We next state a lower bound for regret in the standard MAB setting. Namely, we bound $\tilde{R}_n(g)$, which (we recall) is the regret incurred when using decision function $g$ for $n$ arm pulls with full reward observations (see \eqref{eqRegretMAB}).
\begin{lem} \label{lemLower}
Let $g = \{ g_t \}_{t \in \N}$ satisfy Assumptions \ref{assMono}-\ref{assConv1}. Then for any $\delta \in (0, \mu_K \wedge 1-\mu_1 )$ independent of $n$, the regret \eqref{eqRegretMAB} satisfies
\begin{equation}
\liminf_{n \rightarrow \infty} \frac{\tilde{R}_n(g) }{\sum_{k=2}^K \Delta_k S_{ \ceil{(1-\delta)n} } ( \mu_k - \delta , \mu_1 + \delta  ) } \geq 1 .
\end{equation}
\end{lem}
\begin{proof}[Proof sketch]
In Appendix \ref{secProofLower}, we prove the stronger result $\lim_{n \rightarrow \infty} \P ( \mathcal{E}_n ) = 0$, where
\begin{equation}
\mathcal{E}_n = \left\{ \sum_{k=2}^K \Delta_k \tilde{T}_k(n) <  \sum_{k=2}^K \Delta_k S_{{(1-\delta)n}} ( \mu_k - \delta , \mu_1 + \delta  ) \right\}  .
\end{equation}
The proof of $\P ( \mathcal{E}_n ) \rightarrow 0$ proceeds in three steps. We first show $\mathcal{E}_n \subset \mathcal{F}_n \cup \mathcal{G}_n$, where
\begin{gather}
\mathcal{F}_n = \left\{ \max_{s \geq (1-2\delta) } g_n ( \hat{\mu}_{1,s} , s ) > \mu_1 + \delta  \right\} , \\
\mathcal{G}_{n,k} = \left\{ \min_{s \leq S_{ {(1-\delta)n}} (\mu_k-\delta,\mu_1+\delta) } g_{{(1-\delta)n}} ( \hat{\mu}_{k,s},s) \leq \mu_1 + \delta  \right\}  , \quad \mathcal{G}_n = \cup_{k=2}^K \mathcal{G}_{n,k} .
\end{gather}
To do so, we assume instead that $\mathcal{E}_n \cap \mathcal{F}_n^C \cap \mathcal{G}_n^C \neq \emptyset$ and derive a contradiction as follows. First, note $\mathcal{E}_n$ implies $\tilde{T}_k(n) < S_{(1-\delta)n} ( \mu_k - \delta , \mu_1 + \delta )$ for some $k \geq 2$. Combined with $\mathcal{G}_{n,k}^C$ and Assumption \ref{assMono}, this implies that for all $t \geq (1-\delta) n$,
\begin{equation}\label{eqNewEvents1Main}
g_t ( \hat{\mu}_{ k , \tilde{T}_k (t-1) } , \tilde{T}_k (t-1) ) \geq \min_{s \leq S_{ {(1-\delta)n}} (\mu_k-\delta,\mu_1+\delta) } g_{{(1-\delta)n}} ( \hat{\mu}_{k,s},s) > \mu_1 + \delta .
\end{equation} 
On the other hand, $\mathcal{E}_n$ and Assumption \ref{assConv1} imply $\sum_{k=2}^K \tilde{T}_k(n) = o ( n )$, so $\tilde{T}_1 ( (1 - \delta) n ) \geq (1-2\delta)n$ for large $n$. Combined with $\mathcal{F}_n^C$ and Assumption \ref{assMono}, we have, for any $t \geq (1-\delta) n$,
\begin{equation}\label{eqNewEvents2Main}
g_t ( \hat{\mu}_{1, \tilde{T}_1(t-1) } , \tilde{T}_1(t-1) ) \leq \max_{s \geq (1-2\delta) } g_n ( \hat{\mu}_{1,s} , s ) \leq \mu_1 + \delta .
\end{equation}
Comparing \eqref{eqNewEvents1Main} and \eqref{eqNewEvents2Main}, the index of arm $k$ exceeds that of arm $1$ for all $t \geq (1-\delta)n$. Thus, by definition of the policy \eqref{eqTildeAt}, arm $1$ is never pulled between times $(1-\delta)n$ and $n$, which contradicts the bound $\sum_{k=2}^K \tilde{T}_k(n) = o ( n )$ stated above. Therefore, $\mathcal{E}_n \cap \mathcal{F}_n^C \cap \mathcal{G}_n^C$ cannot occur, and subsequently $\mathcal{E}_n  \subset \mathcal{F}_n  \cup \mathcal{G}_n$, for large $n$.

Next, we show $\P ( \mathcal{F}_n ) \rightarrow 0$, which amounts to bounding $\P ( g_n ( \hat{\mu}_{1,s} , s ) > \mu_1 + \delta)$ for $s \geq (1-2\delta) n$. For this, we first note $g_n ( \hat{\mu}_{1,s} , s ) > \mu_1 + \delta$ implies $\hat{\mu}_{1,s} > \mu_1 + \delta / 2$: if not, $g_n ( \mu_1 + \delta / 2 , s ) > \mu_1 + \delta$ by Assumption \ref{assMono}, so $S_n ( \mu_1 + \delta / 2 , \mu_1 + \delta ) > s \geq (1-2\delta)n$ by Definition \ref{defnConv}, which violates Assumption \ref{assConv1} for large $n$. Thus, our task is reduced to bounding $\P ( \hat{\mu}_{1,s} > \mu_1 + \delta / 2 )$, for which we use the Hoeffding bound.
 
Finally, we show  $\P ( \mathcal{G}_{n,k} ) \rightarrow 0$. Here we aim to bound $\P ( g_{{(1-\delta)n}} ( \hat{\mu}_{k,s},s) \leq \mu_1 + \delta )$ for $s \leq S_{(1-\delta)n} ( \mu_k - \delta , \mu_1 + \delta )$. For such $s$, $g_{(1-\delta)n} ( \mu_k - \delta , \mu_1 + \delta ) > \mu_1 + \delta$ by Definition \ref{defnConv}, so $g_{{(1-\delta)n}} ( \hat{\mu}_{k,s},s) \leq \mu_1 + \delta$, which implies  $\hat{\mu}_{k,s} \leq \mu_k - \delta$ (else, we violate Assumption \ref{assMono}). The probability of the latter event can again be bounded via Hoeffding. However, there is one subtlety: the resulting bound is too large for small $s$. We address this using Assumption \ref{assUpper}, similar to how we resolved the small $s_1$ issue in the Lemma \ref{lemUpperFinite} proof sketch.

To summarize, we have argued $\mathcal{E}_n \subset \mathcal{F}_n \cup ( \cup_{k=2}^K \mathcal{G}_{n,k} )$ (for all large $n$), $\P ( \mathcal{F}_n ) \rightarrow 0$, and $\P ( \mathcal{G}_{n,k} ) \rightarrow 0$ (for all $k$). The desired result $\P(\mathcal{E}_n) \rightarrow 0$ follows from the union bound.
\end{proof}

Combining the lemmas and using Assumption \ref{assConv2}, we obtain the following theorem.
\begin{thm}[Main result] \label{thmMain}
Let $g = \{ g_t \}_{t \in \N}$ satisfy Assumptions \ref{assMono}-\ref{assConv2}. Then
\begin{equation}
\limsup_{n \rightarrow \infty} \frac{R_n ( \pi^a(g) , \pi^m )}{ \tilde{R}_n(g) } \leq 1 .
\end{equation}
\end{thm}
\begin{proof}
See Appendix \ref{secProofMain}.
\end{proof}

\subsection{Finite-horizon bounds for \texttt{UCB1} and \texttt{KL-UCB} variants} \label{secResultsFinite}

At the level of generality of Assumptions \ref{assMono}-\ref{assConv2}, we can only bound regret in terms of the $S_t^g(x_1,x_2)$ function from Definition \ref{defnConv}. If we specialize to \texttt{UCB1} and \texttt{KL-UCB}, Proposition \ref{propAssumptions} and Lemma \ref{lemUpperFinite} yield bounds depending only on the horizon and the arm means. We begin by stating the bound for \texttt{UCB1} (the proof details can be found in Appendix \ref{secProofCor}).

\begin{cor} \label{corFiniteUCB1}
Let $n \in \{2,3,\ldots,\}$ and define $C_n = 2 + (\log n )^{1/3}$. Then
\begin{alignat}{2}
R_n ( \pi^a ( g^{\texttt{UCB1}} ) , \pi^m ) & \leq \min \Bigg\{ && \sum_{k=2}^K  \Bigg( \frac{8 \log n}{ \Delta_k  } + \frac{160}{\Delta_k}  + \frac{19 \Delta_k}{6} \log_2 \Bigg(\ceil*{\frac{8 \log n}{ \Delta_k^2 }} \vee \frac{64}{\Delta_k^2} \Bigg) \Bigg)  , \\
& && \sum_{k=2}^K   \Bigg( \frac{2 \log n}{ \Delta_k} + \frac{8  ( 1 - \frac{1}{C_n} ) \log n }{ \Delta_k  ( 1 -\frac{2}{C_n} )^2 C_n } + \frac{10 C_n^2}{\Delta_k}  \\
& && \quad\quad\quad\quad +  \frac{ 13 \Delta_k }{4} \log_2 \Bigg(\ceil*{\frac{2 \log n}{ \Delta_k^2 ( 1 - \frac{2}{C_n} )^2 }} \vee \frac{4 C_n^2}{\Delta_k^2} \Bigg)  \Bigg) \Bigg\}  \\
& = \min \Bigg\{ && \sum_{k=2}^K \frac{ 8 \log n }{ \Delta_k } + O \Big( \frac{K \log \log n}{\Delta_2} \Big)  , \sum_{k=2}^K \frac{ 2 \log n }{ \Delta_k } + O \Big( \frac{K ( \log n )^{2/3}}{\Delta_2} \Big)  \Bigg\} . 
\end{alignat}
\end{cor}
The $\log n$ term in the second bound is precisely the limiting lower bound from Lemma \ref{lemLower}. The first bound is worse for large $n$, but it more closely resembles the bound $\tilde{R}_n( g^{\texttt{UCB1}} ) \leq \sum_{k=2}^K 8 \log (n) / \Delta_k + O ( 1 )$ from \cite{auer2002finite}. In particular, under the one-bit feedback restriction, we recover the $O (\log n )$ term while increasing the $O(1)$ term to $O(\log \log n)$.

For \texttt{KL-UCB}, the $\log n$ term similarly matches the limiting lower bound from Lemma \ref{lemLower}.
\begin{cor} \label{corFiniteKL}
Let $n \in \{2,3,\ldots,\}$ and define $C_n = 2 + ( \log n )^{1/3}$. Then
\begin{align}
R_n ( \pi^a ( g^{\texttt{KL-UCB}} ) , \pi^m ) & \leq \sum_{k=2}^K   \Bigg( \frac{\Delta_k \log f(n)}{ d ( \mu_k , \mu_1 ) } + \frac{ (\frac{ 1 }{ \mu_k } + \frac{1}{1-\mu_1} )^2 \log f(n) }{ C_n \Delta_k ( 2 - \frac{1}{C_n} )^2  }   + \frac{10 C_n^2}{\Delta_k}  \\
& \quad\quad\quad\quad + \frac{13 \Delta_k}{4} \log_2 \Bigg(\ceil*{\frac{\log f(n)}{ 2 \Delta_k^2 ( 1 - \frac{2}{C_n} )^2  }} \vee \frac{4 C_n^2}{\Delta_k^2} \Bigg) \Bigg) \\
& = \sum_{k=2}^K  \frac{\Delta_k \log f(n)}{ d(\mu_k,\mu_1)}  + O \Big( \frac{K ( \log n )^{2/3}}{\Delta_2} \Big).
\end{align}
\end{cor}

\begin{rem}
In the proofs of Corollaries \ref{corFiniteUCB1} and \ref{corFiniteKL}, we choose values of $\delta$ which are not necessarily the minimizers in Lemma \ref{lemUpperFinite}. These values are chosen to ensure the logarithmic terms match the lower bounds from Lemma \ref{lemLower} and the sublogarithmic terms have reasonably simple expressions. Thus, the constants multiplying sublogarithmic terms are not optimized, and these terms may dominate for small $n$. However, we will soon show the empirical performance of our policies is competitive with the MAB policies uniformly in $n$. 
\end{rem}

\section{Experiments} \label{secExp}

\begin{figure}
\centering
\includegraphics[width=6in]{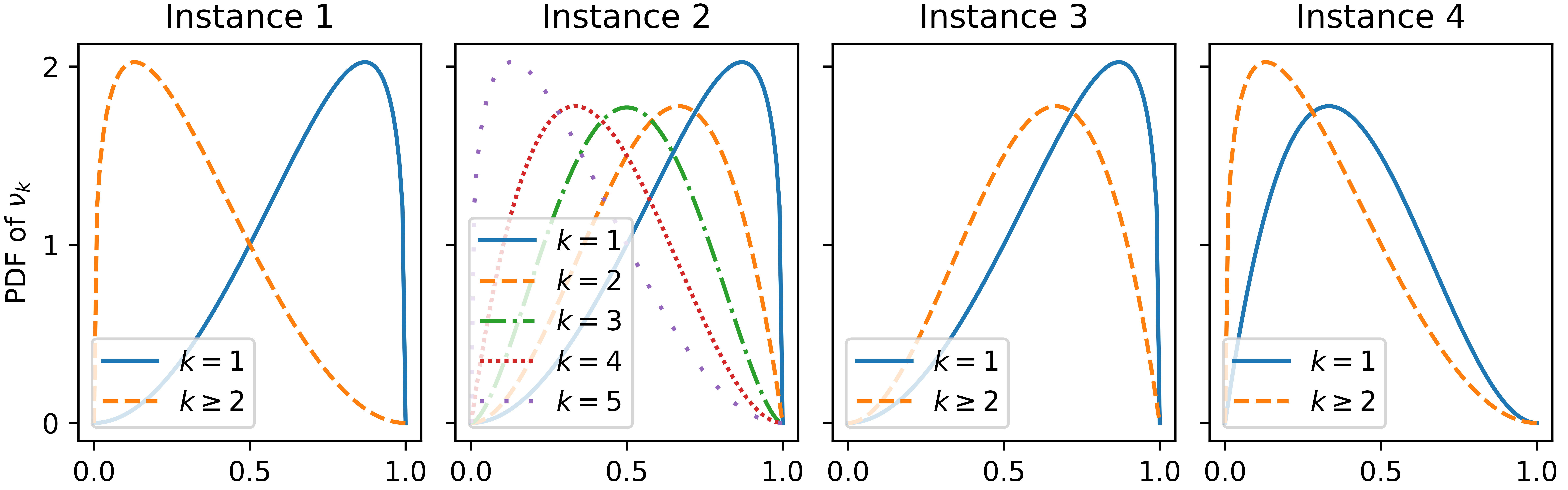}
\caption{PDFs of reward distributions.} \label{figInstances}
\end{figure}

Before closing, we illustrate the performance of our policies numerically. We set $K=5$ and $\nu_k = \text{Beta}(a_k,b_k)$ for four choices of $a = (a_k)_{k=1}^5$ and $b = (b_k)_{k=1}^5$. This results in  four problem instances with the PDFs of $\nu_k$ shown in Figure \ref{figInstances}. The instances are as follows:
\begin{enumerate}
\item Large arm gap: $a = (3,1.3,\ldots,1.3), b = (1.3,3,\ldots,3)$.
\item Small arm gap, distinct arms: $a = (3,3,2.7,2,1.3), b = (1.3,2,2.7,3,3)$.
\item Small gap, identical suboptimal arms, high rewards: $a = (3,\ldots,3), b = (1.3,2,\ldots,2)$.
\item Small gap, identical suboptimal arms, low rewards: $a = (2,1.3,\ldots,1.3), b = (3,\ldots,3)$.
\end{enumerate}

\begin{figure}
\centering
\includegraphics[width=6in]{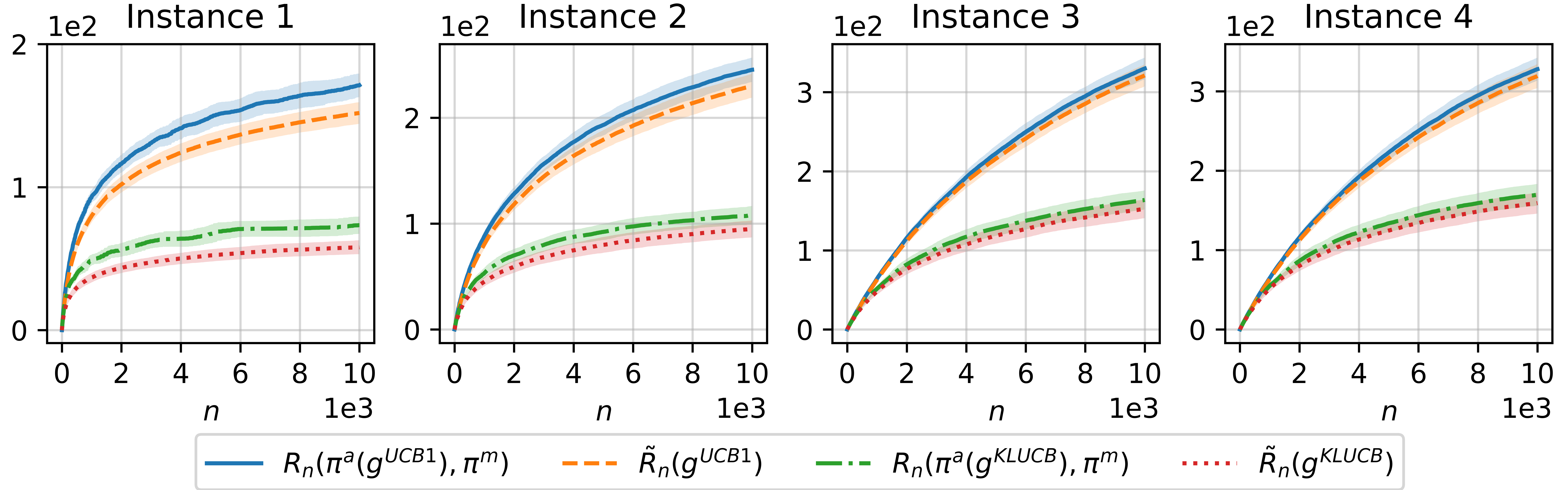}
\caption{Regret for proposed policies and corresponding MAB policies.} \label{figRegret}
\end{figure}

In Figure \ref{figRegret}, we plot the estimated regret (mean $\pm$ standard deviation across 100 trials) for the \texttt{UCB1} and \texttt{KL-UCB} variants of our policy ($R_n(\pi^a( g^{\texttt{UCB1} } ) , \pi^m )$ and $R_n(\pi^a( g^{\texttt{KL-UCB} } ) , \pi^m )$, respectively) at each $n \in [10^4]$. For comparison, we show estimated regret for \texttt{UCB1} and \texttt{KL-UCB} in the standard MAB setting ($\tilde{R}_n( g^{\texttt{UCB1}})$ and $\tilde{R}_n( g^{\texttt{KL-UCB} } )$, respectively). For both choices of $g$, $R_n(\pi^a( g),\pi^m)$ closely tracks $\tilde{R}_n(g)$. We emphasize that, while Theorem \ref{thmMain} ensures these quantities are asymptotically equivalent, Figure \ref{figRegret} shows they are close uniformly in $n$. These behaviors are consistent across problem instances, which are roughly ordered in terms of regret (observe, for example, that the large gap instance quickly reaches its logarithmic regret regime). Finally, we note that \texttt{KL-UCB} outperforming \texttt{UCB1} is unsurprising based on results from \cite{cappe2013kullback,garivier2011kl,maillard2011finite}; however, it is noteworthy that $R_n(\pi^a( g^{\texttt{KL-UCB} } ) , \pi^m )$ is significantly smaller than $\tilde{R}_n( g^{\texttt{UCB1}})$ across $n$ and problem instances.

\section*{Acknowledgements}

This work was partially supported by ONR Grant N00014-19-1-2566, NSF Grants SATC 1704778 and CCF 1934986, and ARO Grants W911NF-17-1-0359 and W911NF-19-1-0379.

\bibliographystyle{plain}
\bibliography{references}

\newpage \appendix

\section{Proof of Lemma \ref{lemUpperFinite}} \label{secProofUpperFinite}

As discussed in the proof sketch, we prove \eqref{eqNumPlaysMain} using ideas from \cite[Chapter 10]{lattimore2020bandit}. In the case $n \leq K$, Algorithm \ref{algArm} ensures $T_k(n) \leq 1$ (arms are played in round robin once), and the right side of \eqref{eqNumPlaysMain} exceeds $1$, so the bound is immediate. In the case $n > K$, we first define
\begin{gather}
\sigma_1 = \min \left\{ t \in \N : \min_{s \in [n]} g_t ( \hat{\mu}_{1,s},s) \geq \mu_1 - \delta \right\}  , \label{eqTau1Defn} \\
\sigma_k = \sum_{s=1}^n 1 ( g_n (\bar{\mu}_{k,s} , \eta(s)   ) \geq \mu_1 - \delta ) , \label{eqTauKDefn}
\end{gather}
where $1(\cdot)$ denotes the indicator function and $\bar{\mu}_{k,s}$ is defined in \eqref{eqDefnMuBar}. We will show
\begin{gather}
\E T_k(n)  \leq \E \sigma_1 + \E \sigma_k , \label{eqUpFin1}  \\
\E \sigma_1 \leq \frac{2}{\delta^2} , \label{eqUpFin2} \\
\E \sigma_k \leq S_n(\mu_k+\delta,\mu_1-\delta) \vee \frac{4}{\delta^2} + 3 \log_2 \left(  S_n(\mu_k+\delta,\mu_1-\delta) \vee \frac{4}{\delta^2}  \right) + \frac{4}{\delta^2} , \label{eqUpFin3}
\end{gather}
from which \eqref{eqNumPlaysMain} follows. For \eqref{eqUpFin1}, we first claim that for any $t \in \{K+1,\ldots,n\}$,
\begin{equation}\label{eqIndicatorIneq}
A_t = k , t \geq \sigma_1 , T_k(t-1) = s \quad \Rightarrow \quad g_n(  \bar{\mu}_{k,s} , \eta(s)   )  \geq \mu_1 - \delta .
\end{equation}
To prove \eqref{eqIndicatorIneq}, we assume $A_t = k , t \geq \sigma_1 , T_k(t-1) = s$ and observe
\begin{align}
g_n( \bar{\mu}_{k,s} , \eta(s)   ) & \geq g_t(  \bar{\mu}_{k,s}, \eta(s)   )  = g_t(  \bar{\mu}_{k,T_k(t-1)} , \eta(T_k(t-1))   ) \\
& \geq g_t(  \bar{\mu}_{1,T_1(t-1)} , \eta(T_1(t-1))   ) \geq g_t(   \hat{\mu}_{1,\eta(T_1(t-1)) } , \eta(T_1(t-1))   ) \\
& \geq \min_{s' \in [n]} g_t(   \hat{\mu}_{1, s' } , s'  )   \geq \min_{s' \in [n]} g_{\sigma_1}(   \hat{\mu}_{1, s' } , s'  ) \geq \mu_1 - \delta ,
\end{align}
where the (in)equalities follow from Assumption \ref{assMono}, $T_k(t-1)=s$, the arm policy \eqref{eqArmPolicy} (recall $t > K$), Proposition \ref{propMessageAccuracy} and Assumption \ref{assMono}, $\eta(T_1(t-1)) \in [n]$, $t \geq \sigma_1$ and Assumption \ref{assMono}, and definition of $\sigma_1$ \eqref{eqTau1Defn}, respectively. Thus, \eqref{eqIndicatorIneq} holds, from which we obtain
\begin{align}
T_k(n) & = 1 + \sum_{t=K+1}^n 1 ( A_t = k , t \leq \sigma_1 - 1 ) + \sum_{s=1}^n \sum_{t=K+1}^n 1 ( A_t = k , t \geq \sigma_1 , T_k(t-1) = s ) \\
& \leq \sigma_1 + \sum_{s=1}^n 1 ( g_n(  \bar{\mu}_{k,s} , \eta(s)   )  \geq \mu_1 - \delta ) \sum_{t=1}^n 1 ( A_t = k , T_k(t-1) = s ) \\
& \leq \sigma_1 + \sum_{s=1}^n 1 ( g_n(  \bar{\mu}_{k,s} , \eta(s)   )  \geq \mu_1 - \delta )  = \sigma_1 + \sigma_k ,
\end{align}
where the first equality follows from Algorithm \ref{algArm}, the first inequality uses \eqref{eqIndicatorIneq}, and the final equality holds by definition \eqref{eqTauKDefn}. Taking expectation on both sides yields \eqref{eqUpFin1}.

To prove \eqref{eqUpFin2}, we first let $t \in \N$ and use the union bound to write
\begin{equation}\label{eqUpFin2union}
\P ( \sigma_1 > t ) \leq  \sum_{s=1}^{\infty} \P ( g_t ( \hat{\mu}_{1,s} , s ) < \mu_1 - \delta ) . 
\end{equation}
Now fix $s \in \N$. Note $g_t ( \hat{\mu}_{1,s} , s ) < \mu_1 - \delta $ implies $g_t ( \hat{\mu}_{1,s} , s ) < 1$, so by Assumption \ref{assUpper},
\begin{equation}
g_t ( \hat{\mu}_{1,s} , s )  > \hat{\mu}_{1,s} , \quad d ( \hat{\mu}_{1,s} , g_t ( \hat{\mu}_{1,s} , s ) ) \geq \frac{\log  f(t) }{s} .
\end{equation}
Hence, $g_t ( \hat{\mu}_{1,s} , s ) < \mu_1 - \delta $ ensures $\hat{\mu}_{1,s} < g_t ( \hat{\mu}_{1,s} , s ) < \mu_1 - \delta < \mu_1$, so by Proposition \ref{propKL},
\begin{equation}
d ( \hat{\mu}_{1,s} , \mu_1 ) \geq d ( \hat{\mu}_{1,s} , g_t ( \hat{\mu}_{1,s} ,s ) ) + d ( g_t ( \hat{\mu}_{1,s}  ,s ), \mu_1 - \delta ) + d ( \mu_1 - \delta , \mu_1 ) > \frac{ \log f(t) }{ s } + 2 \delta^2 .
\end{equation}
We have thus established the following implication:
\begin{equation}
g_t ( \hat{\mu}_{1,s} , s ) < \mu_1 - \delta  \quad \Rightarrow \quad d ( \hat{\mu}_{1,s} , \mu_1 ) > \frac{ \log f(t) }{ s } + 2 \delta^2 , \hat{\mu}_{1,s} < \mu_1 .
\end{equation}
Plugging into \eqref{eqUpFin2union}, we obtain
\begin{equation}
\P ( \sigma_1 > t ) \leq \sum_{s=1}^{\infty} \P \left( d ( \hat{\mu}_{1,s} , \mu_1 ) > \frac{ \log f(t) }{ s } + 2 \delta^2 , \hat{\mu}_{1,s} < \mu_1 \right) .
\end{equation}
The remainder of the proof follows as in \cite[Lemma 10.7]{lattimore2020bandit}.

Finally, we prove \eqref{eqUpFin3}. As in the sketch, define $u = S_n (  \mu_k + \delta , \mu_1 - \delta  ) \vee \frac{4}{\delta^2}$ and let $s \in \N$ s.t.\ $s \geq u + 3 \log_2 u$. Then by Propositions \ref{propMessageAccuracy} and \ref{propAlphaPsiMain}, we have (similar to \eqref{eqMuKEtaKMain} in the sketch),
\begin{equation}\label{eqMuKEtaKapp}
\eta(s) \geq \frac{s}{2} \vee u , \quad \bar{\mu}_{k,s} < \hat{\mu}_{k,\eta(s)} + \frac{\delta}{2} .
\end{equation}
Now suppose $g_n (  \bar{\mu}_{k,s} , \eta(s)  ) \geq \mu_1 - \delta$ and $\hat{\mu}_{k,\eta(s)} < \mu_k + \delta / 2$. Then
\begin{equation}
\mu_1 - \delta \leq g_n (  \bar{\mu}_{k,s} , \eta(s)  ) < g_n ( \mu_k + \delta , S_n ( \mu_k+\delta , \mu_1 - \delta ) ) \leq \mu_1 - \delta ,
\end{equation}
where the first inequality holds by assumption; the second uses \eqref{eqMuKEtaKapp}, $\hat{\mu}_{k,\eta(s)} \leq \mu_k + \delta / 2$, and Assumption \ref{assMono}; and the third uses Definition \ref{defnConv}. This is a contradiction, so we conclude
\begin{equation}
g_n (  \bar{\mu}_{k,s} , \eta(s)  ) \geq \mu_1 - \delta \quad \Rightarrow \quad \hat{\mu}_{k,\eta(s)} \geq \mu_k + \frac{\delta}{2}  .
\end{equation}
Thus, by the Hoeffding bound (see Proposition \ref{propHoeffding}) and \eqref{eqMuKEtaKapp},
\begin{equation}
\P ( g_n (  \bar{\mu}_{k,s} , \eta(s)  ) \geq \mu_1 - \delta ) \leq \P ( \hat{\mu}_{k,\eta(s)} \geq \mu_k + \delta / 2  ) \leq e^{ - \delta^2\eta(s) /2 } \leq  e^{ -  \delta^2 s / 4  } .
\end{equation}
Since this argument holds for any $s \geq u + 3 \log_2 u$, we conclude
\begin{equation}
\sum_{s = \ceil{ u + 3 \log_2 u } }^{\infty}  \P ( g_n (  \bar{\mu}_{k,s} , \eta(s)  ) \geq \mu_1 - \delta )  \leq \sum_{s=1}^{\infty} e^{ - \delta^2 s / 4 } \leq \int_{s=0}^{\infty}  e^{ - \delta^2 s / 4 } ds  = \frac{4}{\delta^2} . 
\end{equation}
The inequality \eqref{eqUpFin3} then follows by definition \eqref{eqTauKDefn}:
\begin{equation}
\E \sigma_k = \sum_{s=1}^{\infty} \P ( g_n ( \bar{\mu}_{k,s} , \eta(s) ) \geq \mu_1 - \delta ) \leq u + 3 \log_2 u + \frac{4}{\delta^2} .
\end{equation}

\section{Proof of Lemma \ref{lemLower}} \label{secProofLower}

Fix $\delta \in (0, \min \{ \mu_K , 1 - \mu_1 \})$ and (as in the proof sketch) define the events
\begin{allowdisplaybreaks}
\begin{gather}
\mathcal{E}_n = \left\{ \sum_{k=2}^K \Delta_k \tilde{T}_k(n) <  \sum_{k=2}^K \Delta_k S_{\ceil{(1-\delta)n}} ( \mu_k - \delta , \mu_1 + \delta  ) \right\} , \\
\mathcal{F}_n = \left\{ \max_{s \in \{ \ceil{ (1-2\delta) n}, \ldots , n \}} g_n ( \hat{\mu}_{1,s} , s ) > \mu_1 + \delta  \right\} , \\
\mathcal{G}_{n,k} = \left\{ \min_{s \in \{1,\ldots, S_{ \ceil{(1-\delta)n}} (\mu_k-\delta,\mu_1+\delta) \}} g_{\ceil{(1-\delta)n}} ( \hat{\mu}_{k,s},s) \leq \mu_1 + \delta  \right\}  , \quad \mathcal{G}_n = \cup_{k=2}^K \mathcal{G}_{n,k} .
\end{gather}
\end{allowdisplaybreaks}
Then by the regret decomposition \eqref{eqRegretMAB},
\begin{equation}
\tilde{R}_n(g) = \E \sum_{k=2}^K \Delta_k \tilde{T}_k(n) \geq \E \sum_{k=2}^K \Delta_k \tilde{T}_k(n) 1 ( \mathcal{E}_n^C ) \geq \sum_{k=2}^K \Delta_k S_{\ceil{(1-\delta)n}} ( \mu_k - \delta , \mu_1 + \delta  ) \P ( \mathcal{E}_n^C ) .
\end{equation}
Thus, we can prove the lemma by showing $\P ( \mathcal{E}_n ) \rightarrow 0$ as $n \rightarrow \infty$. In particular, we will prove the three results discussed in the proof sketch:
\begin{gather}
\exists\ N \in \N\ s.t.\ \mathcal{E}_n \subset \mathcal{F}_n \cup \mathcal{G}_n\ \forall\ n \geq N , \label{eqNewEvents} \\
\lim_{n \rightarrow \infty} \P ( \mathcal{F}_n  ) = 0 , \label{eqLimFn} \\
\lim_{n \rightarrow \infty} \P ( \mathcal{G}_{n,k}  ) = 0\ \forall\ k \in [K] . \label{eqLimGn}
\end{gather}

We begin with \eqref{eqNewEvents}. As in the proof sketch, we will argue $\mathcal{E}_n \cap \mathcal{F}_n^C \cap \mathcal{G}_n^C$ results in a contradiction for large $n$. Toward this end, first note that when $n$ is large and $\mathcal{E}_n$ holds,
\begin{equation}\label{eqAssConv1app}
\sum_{k=2}^K \tilde{T}_k(n) \leq \frac{1}{\Delta_2} \sum_{k=2}^K \Delta_k \tilde{T}_k(n) < \frac{1}{\Delta_2} \sum_{k=2}^K \Delta_k S_{\ceil{(1-\delta)n}} ( \mu_k - \delta , \mu_1 + \delta  ) \leq \delta n -1 ,
\end{equation}
where the final inequality holds by Assumption \ref{assConv1}. Thus, for large $n$, $\mathcal{E}_n$ implies
\begin{align}
\tilde{T}_1 ( \ceil{ (1-\delta) n } ) & = n - \sum_{k=2}^K \tilde{T}_k(n) - \left( \tilde{T}_1(n) - \tilde{T}_1 ( \ceil{ (1-\delta) n } ) \right) \\
& > n - \delta n + 1 - ( n - \ceil{ (1-\delta) n } ) \geq \ceil{ (1-2\delta)n} .
\end{align}
Consequently, when $n$ is large, $\mathcal{E}_n$ and $\mathcal{F}_n^C$ imply that for any $t \in \{ 1 + \ceil{(1-\delta)n} , \ldots , n \}$,
\begin{equation}\label{eqNewEvents1}
g_t ( \hat{\mu}_{1, \tilde{T}_1(t-1)} , \tilde{T}_1(t-1) ) \leq \max_{s \in \{ \ceil{ (1-2\delta) n}, \ldots , n \}} g_n ( \hat{\mu}_{1,s} , s ) \leq \mu_1 + \delta ,
\end{equation}
where we also used monotonicity of $\tilde{T}_1(\cdot)$ and Assumption \ref{assMono}. On the other hand, $\mathcal{E}_n$ implies $\tilde{T}_{k_n}(n) < S_{ \ceil{(1-\delta)n} } ( \mu_{k_n} - \delta , \mu_1 + \delta )$ for some $k_n \in \{2,\ldots,K\}$ (else, $\mathcal{E}_n^C$ holds). Thus, for any $t \in \{ 1+\ceil{(1-\delta)n} , \ldots , n \}$, $\mathcal{E}_n$ and $\mathcal{G}_n^C$ imply
\begin{align}
g_t ( \hat{\mu}_{k_n,\tilde{T}_{k_n}(t-1)} , \tilde{T}_{k_n}(t-1) ) & \geq g_{ \ceil{ (1-\delta) n } } ( \hat{\mu}_{k_n,\tilde{T}_{k_n}(t-1)} , \tilde{T}_{k_n}(t-1) ) \\
& \geq \min_{s \in \{1,\ldots, S_{ \ceil{(1-\delta)n}} (\mu_{k_n}-\delta,\mu_1+\delta) \}} g_{\ceil{(1-\delta)n}} ( \hat{\mu}_{k_n,s},s) > \mu_1 + \delta , \label{eqNewEvents2}
\end{align}
where we also used monotonicity of $\tilde{T}_{k_n}(\cdot)$ and Assumption \ref{assMono}. Comparing \eqref{eqNewEvents1} and \eqref{eqNewEvents2}, we see that for large $n$, $\mathcal{E}_n \cap \mathcal{F}_n^C \cap \mathcal{G}_n^C$ implies that for any $t \in \{ 1+\ceil{(1-\delta)n} , \ldots , n \}$,
\begin{equation}
g_t ( \hat{\mu}_{k_n,\tilde{T}_{k_n}(t-1)} , \tilde{T}_{k_n}(t-1) ) > g_t ( \hat{\mu}_{1,\tilde{T}_{1}(t-1)} , \tilde{T}_{1}(t-1) ) .
\end{equation}
By the policy \eqref{eqTildeAt}, this means $\tilde{A}_t \neq 1\ \forall\ t \in \{ 1+\ceil{(1-\delta)n} , \ldots , n \}$, and consequently,
\begin{equation}
\sum_{k=2}^K \tilde{T}_k(n)  \geq n - \ceil{(1-\delta)n} \geq \delta n - 1 ,
\end{equation}
which contradicts \eqref{eqAssConv1app}. Thus, $\mathcal{E}_n \cap \mathcal{F}_n^C \cap \mathcal{G}_n^C = \emptyset$ for large $n$, completing the proof of \eqref{eqNewEvents}.

For \eqref{eqLimFn}, we claim (and return to prove) that for $n$ sufficiently large and $s \geq \ceil{(1-2\delta) n}$,
\begin{equation}\label{eqFnToZeroKeyImp}
g_n ( \hat{\mu}_{1,s} , s ) > \mu_1 + \delta  \quad \Rightarrow \quad \hat{\mu}_{1,s} > \mu_1 + \delta / 2 .
\end{equation}
The proposition then follows from the union bound and Proposition \ref{propHoeffding}:
\begin{equation}
\P ( \mathcal{F}_n )  \leq \sum_{s =\ceil{(1-2\delta) n}}^{\infty} \P ( \hat{\mu}_{1,s} > \mu_1 + \delta / 2  ) \leq \sum_{s =\ceil{(1-2\delta) n}}^{\infty} e^{- s \delta^2 / 2 } \xrightarrow[n \rightarrow \infty]{} 0 .
\end{equation}
To prove \eqref{eqFnToZeroKeyImp}, suppose instead that $g_n ( \hat{\mu}_{1,s} , s ) > \mu_1 + \delta$ and $\hat{\mu}_{1,s} \leq \mu_1 + \delta / 2$. Then
\begin{equation}
g_n ( \mu_1 + \delta / 2 , \ceil{ (1-2\delta) n } ) \geq g_n ( \hat{\mu}_{1,s} , s ) > \mu_1 + \delta ,
\end{equation}
where we used Assumption \ref{assMono}. Hence, by Definition \ref{defnConv}, $S_n ( \mu_1 + \delta/2 , \mu_1+ \delta ) >  (1-2\delta) n$, which contradicts Assumption \ref{assConv1} for large $n$.

For \eqref{eqLimGn}, we let $s \in \N$ s.t.\ $s <  \frac{\log  f( (1-\delta) n ) }{- \log ( 1 - \mu_1 - \delta )} $. Then since $\hat{\mu}_{k,s} \geq 0$, Assumption \ref{assMono} implies $g_{\ceil{(1-\delta)n}} ( 0,s) \leq g_{\ceil{(1-\delta)n}} ( \hat{\mu}_{k,s},s)$. Thus, if $g_{\ceil{(1-\delta)n}} ( \hat{\mu}_{k,s},s) \leq \mu_1 + \delta$, we obtain
\begin{align}
\log \left(  \frac{1}{1 - \mu_1 - \delta} \right) & \geq \log \left( \frac{1}{ 1 -g_{\ceil{(1-\delta)n}} ( 0,s) } \right) = d ( 0 ,g_{\ceil{(1-\delta)n}} ( 0,s) ) \\
&\geq \frac{\log f( \ceil{(1-\delta)n})}{s} >  \log \left(  \frac{1}{1 - \mu_1 - \delta} \right) ,
\end{align}
where we used Assumption \ref{assUpper}; this is a contradiction. Thus, by the union bound,
\begin{equation}
\P ( \mathcal{G}_{n,k} )  \leq \sum_{s = \ceil{\frac{\log  f ( \ceil{(1-\delta)n } ) }{ - \log ( 1 - \mu_1 - \delta ) }}  }^{  S_{\ceil{(1-\delta)n}}(\mu_k-\delta,\mu_1+\delta)} \P ( g_{\ceil{(1-\delta)n}} ( \hat{\mu}_{k,s},s) \leq \mu_1 + \delta  ) .
\end{equation}
Now fix $s$ as in the sum. Then $g_{\ceil{(1-\delta)n}} ( \mu_k - \delta , s ) \geq \mu_1+\delta$ (since $s <  S_{\ceil{(1-\delta)n}}(\mu_k-\delta,\mu_1+\delta)$), so $g_{\ceil{(1-\delta)n}} ( \hat{\mu}_{k,s},s) < \mu_1 + \delta$ implies $\hat{\mu}_{k,s} \leq \mu_k - \delta$ (else, $g_{\ceil{(1-\delta) n }} ( \mu_k - \delta , s ) < \mu_1 + \delta$). Thus,
\begin{equation}
\P ( \mathcal{G}_{n,k} ) \leq \sum_{s = \ceil{\frac{\log  f ( \ceil{(1-\delta)n } ) }{ - \log ( 1 - \mu_1 - \delta ) }}  }^{ \infty } \P ( \hat{\mu}_{k,s} \leq \mu_k - \delta ) \leq \sum_{s = \ceil{\frac{\log  f ( \ceil{(1-\delta)n } ) }{ - \log ( 1 - \mu_1 - \delta ) }}  }^{\infty } e^{ - 2 \delta^2 s } \xrightarrow[n \rightarrow \infty]{}  0 ,
\end{equation} 
where the second inequality uses Proposition \ref{propHoeffding}.

\section{Proof of Theorem \ref{thmMain}} \label{secProofMain}

Fix $\delta \in (0, \min \{ \Delta_2 / 2 , \mu_K , 1 - \mu_1 \} )$. Note $\delta \in ( 0, \Delta_k / 2 )\ \forall\ k \geq 2$, so by Lemma \ref{lemUpperFinite},
\begin{equation}
R_n ( \pi^a(g), \pi^m ) \leq \sum_{k=2}^K \Delta_k \left( S_n ( \mu_k + \delta , \mu_1 - \delta ) + 3 \log_2 \left( S_n ( \mu_k + \delta , \mu_1 - \delta ) \vee \frac{4}{\delta^2} \right) + \frac{10}{\delta^2} \right) .
\end{equation} 
Dividing both sides by $\sum_{k=2}^K \Delta_k  S_n ( \mu_k + \delta , \mu_1 - \delta )$ and noting these summands are nonnegative,
\begin{equation}\label{eqUpperLogAndConst}
\frac{R_n ( \pi^a(g), \pi^m )}{ \sum_{k=2}^K \Delta_k  S_n ( \mu_k + \delta , \mu_1 - \delta ) } \leq 1 + \sum_{k=2}^K \frac{ 3 \log_2 ( S_n ( \mu_k + \delta , \mu_1 - \delta ) \vee \frac{4}{\delta^2} ) + \frac{10}{\delta^2} }{  S_n ( \mu_k + \delta , \mu_1 - \delta )  } .
\end{equation}
Next, observe that by Assumption \ref{assUpper}, Definition \ref{defnConv}, and Proposition \ref{propKL},
\begin{equation}
S_n ( \mu_k + \delta , \mu_1 - \delta ) \geq \frac{ \log f(n) }{ d ( \mu_k + \delta , g_n ( \mu_k + \delta , S_n ( \mu_k + \delta , \mu_1 - \delta )  ) } \geq \frac{ \log f(n) }{ d( \mu_k + \delta , \mu_1 - \delta ) } \xrightarrow[n \rightarrow \infty]{} \infty .
\end{equation}
Thus, the summation on the right side of \eqref{eqUpperLogAndConst} vanishes, so for all $n$ large,
\begin{equation}
R_n ( \pi^a(g), \pi^m )\leq (1+\delta) \sum_{k=2}^K \Delta_k  S_n ( \mu_k + \delta , \mu_1 - \delta ) .
\end{equation}
Combining with Lemma \ref{lemLower} (which applies by choice of $\delta$), we have, for all $n$ large,
\begin{equation} \label{eqMainResA}
\frac{ R_n ( \pi^a(g), \pi^m )  }{ \tilde{R}_n(g) } \leq \frac{1+\delta}{1-\delta} \left( 1 + \frac{ \sum_{k=2}^K \Delta_k (  S_n ( \mu_k + \delta , \mu_1 - \delta ) - S_{  \ceil{(1-\delta) n} } ( \mu_k - \delta , \mu_1 + \delta ) ) }{ \sum_{k=2}^K \Delta_k  S_{  \ceil{(1-\delta) n} } ( \mu_k - \delta , \mu_1 + \delta ) } \right) . 
\end{equation}
We next observe that by Assumption \ref{assMono} and Definition \ref{defnConv},
\begin{equation}
g_{ \ceil{(1-\delta) n} } ( \mu_k - \delta , S_n ( \mu_k + \delta , \mu_1 - \delta ) ) \leq g_{  n } ( \mu_k + \delta , S_n ( \mu_k + \delta , \mu_1 - \delta ) ) \leq \mu_1 - \delta < \mu_1 + \delta ,
\end{equation}
Thus, $S_n ( \mu_k + \delta , \mu_1 - \delta ) \geq S_{  \ceil{(1-\delta) n} } ( \mu_k - \delta , \mu_1 + \delta )$ by Definition \ref{defnConv}, so similar to \eqref{eqUpperLogAndConst},
\begin{align} \label{eqMainResB}
 & \frac{ \sum_{k=2}^K \Delta_k (  S_n ( \mu_k + \delta , \mu_1 - \delta ) - S_{ \ceil{ (1-\delta) n} } ( \mu_k - \delta , \mu_1 + \delta ) ) }{ \sum_{k=2}^K \Delta_k  S_{  \ceil{ (1-\delta) n} } ( \mu_k - \delta , \mu_1 + \delta ) } \\
& \quad\quad \leq \sum_{k=2}^K \left( \frac{ S_n ( \mu_k + \delta , \mu_1 - \delta )  }{ S_{ \ceil{ (1-\delta) n} } ( \mu_k - \delta , \mu_1 + \delta ) } -1 \right) . 
\end{align}
Combining \eqref{eqMainResA} and \eqref{eqMainResB}, we have shown
\begin{equation}
\frac{ R_n ( \pi^a(g), \pi^m )  }{ \tilde{R}_n(g) }  \leq \frac{1+\delta}{1-\delta}  \left( 1 + \sum_{k=2}^K \left( \frac{ S_n ( \mu_k + \delta , \mu_1 - \delta )  }{ S_{  \ceil{(1-\delta) n} } ( \mu_k - \delta , \mu_1 + \delta ) } -1 \right) \right) . 
\end{equation}
The theorem follows by defining $\{ \delta_i \}_{i \in \N}$ as in Assumption \ref{assConv2}, applying the previous inequality to each $i \in \N$, letting $n \rightarrow \infty$ on both sides, and letting $i \rightarrow \infty$ on the right side.

\section{Proofs of Corollaries \ref{corFiniteUCB1} and \ref{corFiniteKL}} \label{secProofCor}

For Corollary \ref{corFiniteUCB1}, we choose $\delta = \delta_k = \frac{\Delta_k}{4}$ for the $k$-th minimum in Lemma \ref{lemUpperFinite} and use the expression for $S_t(x_1,x_2)$ from Proposition \ref{propAssumptions} and $\ceil{x} < x + 1$ to bound regret by
\begin{equation}
\sum_{k=2}^K  \left( \frac{8 \log n}{ \Delta_k  } + \frac{160}{\Delta_k}  + \Delta_k \left( 3 \log_2 \left(\ceil*{\frac{8 \log n}{ \Delta_k^2 }} \vee \frac{64}{\Delta_k^2} \right) + 1 \right) \right) .
\end{equation}
The first bound then follows from the inequality
\begin{equation} \label{eqPesky1}
1 = \frac{1}{6} \log_2 64 < \frac{1}{6}  \log_2 \left(\ceil*{\frac{8 \log n}{ \Delta_k^2 }} \vee \frac{64}{\Delta_k^2} \right) .
\end{equation}
For the second bound in Corollary \ref{corFiniteUCB1}, we choose $\delta_k = \frac{\Delta_k}{ C_n} \in (0, \frac{\Delta_k}{2} )$ to bound regret by
\begin{align} 
& \sum_{k=2}^K  \left( \frac{2 \log n}{ \Delta_k ( 1 - \frac{2}{C_n} )^2 } + \frac{10 C_n^2}{\Delta_k}  + \Delta_k \left( 3 \log_2 \left(\ceil*{\frac{2 \log n}{ \Delta_k^2 ( 1 - \frac{2}{C_n} )^2 }} \vee \frac{4 C_n^2}{\Delta_k^2} \right)  + 1 \right)  \right) \\
& \quad \leq   \sum_{k=2}^K  \left( \frac{2 \log n}{ \Delta_k ( 1 - \frac{2}{C_n} )^2 } + \frac{10 C_n^2}{\Delta_k}  + \frac{13 \Delta_k}{4} \log_2 \left(\ceil*{\frac{2 \log n}{ \Delta_k^2 ( 1 - \frac{2}{C_n} )^2 }} \vee \frac{4 C_n^2}{\Delta_k^2} \right)  \right) , \label{eqFiniteUCB1init}
\end{align}
where the second inequality is similar to \eqref{eqPesky1} and uses $C_n > 2$. The second bound in Corollary \ref{corFiniteUCB1} follows by adding and subtracting $\sum_{k=2}^K 2 \log ( n ) / \Delta_k$ in \eqref{eqFiniteUCB1init} and using the identity
\begin{equation}
\frac{1}{ ( 1 - \frac{2}{C_n} )^2 } - 1 = \frac{ 4 (1 - \frac{1}{C_n}) }{ C_n ( 1 - \frac{2}{C_n} )^2 } .
\end{equation}

For Corollary \ref{corFiniteKL}, we proceed similarly, choosing $\delta_k = \frac{\Delta_k}{C_n}$ to bound regret by
\begin{align}
& \sum_{k=2}^K  \left( \frac{\Delta_k \log f(n)}{ d ( \mu_k + \frac{\Delta_k}{C_n} , \mu_1 - \frac{\Delta_k}{C_n} ) } + \frac{10 C_n^2}{\Delta_k}  + \frac{13 \Delta_k}{4} \log_2 \left(\ceil*{\frac{\log f(n)}{ d ( \mu_k + \frac{\Delta_k}{C_n} , \mu_1 - \frac{\Delta_k}{C_n} )  }} \vee \frac{4 C_n^2}{\Delta_k^2} \right)   \right) \\
& \quad \leq \sum_{k=2}^K  \left( \frac{\Delta_k \log f(n)}{ d ( \mu_k + \frac{\Delta_k}{C_n} , \mu_1 - \frac{\Delta_k}{C_n} ) } + \frac{10 C_n^2}{\Delta_k}  + \frac{13 \Delta_k}{4} \log_2 \left(\ceil*{\frac{\log f(n)}{ 2 \Delta_k^2 ( 1 - \frac{2}{C_n} )^2  }} \vee \frac{4 C_n^2}{\Delta_k^2} \right)   \right)  ,
\end{align}
where the second inequality follows from Proposition \ref{propKL} (Pinsker's inequality). Thus, adding and subtracting $\sum_{k=2}^K \Delta_k \log ( f(n) ) / d ( \mu_k , \mu_1)$, we aim to show
\begin{equation} \label{eqFiniteKLmain}
\frac{1}{ d ( \mu_k + \frac{\Delta_k}{C_n} , \mu_1 - \frac{\Delta_k}{C_n} ) }  - \frac{1}{ d ( \mu_k , \mu_1) }  \leq  \frac{ (\frac{ 1 }{ \mu_k } + \frac{1}{1-\mu_1} )^2 }{ C_n \Delta_k^2 ( 2 - \frac{1}{C_n} )^2  } 
\end{equation}
Toward this end, we first use Proposition \ref{propKL} to bound the left side of \eqref{eqFiniteKLmain} by
\begin{align}
\frac{1}{ d ( \mu_k + \frac{\Delta_k}{C_n} , \mu_1 - \frac{\Delta_k}{C_n} ) }  - \frac{1}{ d ( \mu_k , \mu_1) }  & = \frac{ d ( \mu_k , \mu_1) - d ( \mu_k + \frac{\Delta_k}{C_n} , \mu_1 - \frac{\Delta_k}{C_n} ) }{ d ( \mu_k , \mu_1)  d ( \mu_k + \frac{\Delta_k}{C_n} , \mu_1 - \frac{\Delta_k}{C_n} ) } \\
& \leq \frac{ d ( \mu_k , \mu_1) - d ( \mu_k + \frac{\Delta_k}{C_n} , \mu_1 - \frac{\Delta_k}{C_n} ) }{ 4 \Delta_k^4 ( 1 - \frac{2}{C_n} )^2 } .
\end{align}
We claim, and will return to prove,
\begin{equation} \label{eqFiniteKLmain2}
d ( \mu_k , \mu_1) - d \left( \mu_k + \frac{\Delta_k}{C_n} , \mu_1 - \frac{\Delta_k}{C_n} \right) \leq \frac{ \Delta_k^2 }{ C_n } \left(\frac{ 1 }{ \mu_k } + \frac{1}{1-\mu_1} \right)^2 .
\end{equation}
The previous two inequalities imply \eqref{eqFiniteKLmain}, which completes the proof of Corollary \ref{corFiniteKL}.

It remains to prove \eqref{eqFiniteKLmain2}. We begin by defining
\begin{gather} \label{eqChiDefn}
\chi_1 = \frac{\Delta_k}{C_n} \left( \log  \frac{\mu_1-\frac{\Delta_k}{C_n}}{ \mu_k + \frac{\Delta_k}{C_n} } + \log \frac{1-\mu_k-\frac{\Delta_k}{C_n}}{ 1-\mu_1 + \frac{\Delta_k}{C_n} } \right)  , \quad \chi_2 = \mu_k \log  \frac{\mu_k(\mu_1-\frac{\Delta_k}{C_n})}{(\mu_k+\frac{\Delta_k}{C_n})\mu_1} , \\
\chi_3 = (1-\mu_k) \log  \frac{(1-\mu_k)(1-\mu_1+\frac{\Delta_k}{C_n})}{(1-\mu_k-\frac{\Delta_k}{C_n})(1-\mu_1)} . 
\end{gather}
Then by definition,
\begin{equation} \label{eqChiRewrite}
d ( \mu_k , \mu_1) - d \left( \mu_k + \frac{\Delta_k}{C_n} , \mu_1 - \frac{\Delta_k}{C_n} \right) = \chi_1 + \chi_2 + \chi_3 .
\end{equation}

To bound $\chi_1$, we simply use $\frac{\Delta_k}{C_n} > 0$, $\log x \leq x-1\ \forall\ x > 0$, and $\mu_1 > \mu_k$ to obtain
\begin{align}
\chi_1 & < \frac{\Delta_k}{C_n} \left( \log  \frac{\mu_1}{ \mu_k  } + \log \frac{1-\mu_k}{ 1-\mu_1 } \right) \leq \frac{\Delta_k}{C_n} \left(   \frac{\mu_1}{ \mu_k  } - 1 +  \frac{1-\mu_k}{ 1-\mu_1 } - 1 \right) \\
& = \frac{\Delta_k}{C_n} \left( \frac{ \Delta_k }{ \mu_k } + \frac{ \Delta_k }{ 1 - \mu_1 } \right) = \frac{ \Delta_k^2 ( 1 - \mu_1 + \mu_k ) }{ C_n \mu_k ( 1 - \mu_ 1 ) } <  \frac{ \Delta_k^2 }{ C_n \mu_k (1-\mu_1) } . \label{eqChi1final}
\end{align}

To bound $\chi_2$ and $\chi_3$, we begin with an intermediate result: $\forall\ p,q \in (0,1), \epsilon \in ( -q , p )$,
\begin{equation} \label{eqFiniteKLpqeps} 
p \log \frac{ p ( q + \epsilon ) }{ ( p - \epsilon ) q } < \epsilon + \frac{ \epsilon p }{ q } + \frac{ 2 \epsilon^2 }{ q(p-\epsilon) } .
\end{equation}
(Note the left side is well-defined by the bounds on $\epsilon$.) To prove \eqref{eqFiniteKLpqeps}, we first compute
\begin{equation}
\frac{ p ( q + \epsilon ) }{ ( p - \epsilon ) q } - 1 = \left( 1 + \frac{\epsilon}{p-\epsilon} \right) \left( 1 + \frac{\epsilon}{q} \right) - 1 = \frac{\epsilon}{q} + \frac{\epsilon}{p-\epsilon} \left( 1 +  \frac{\epsilon}{q} \right)   .
\end{equation}
Hence, using $\log x \leq x-1$, we obtain
\begin{equation}
p \log \frac{ p ( q + \epsilon ) }{ ( p - \epsilon ) q }  \leq \frac{\epsilon p}{q} +  \frac{\epsilon p}{p-\epsilon} \left( 1 +  \frac{\epsilon}{q} \right) .
\end{equation}
The second summand of this upper bound can be rewrriten as
\begin{equation}
\frac{\epsilon p}{p-\epsilon} \left( 1 +  \frac{\epsilon}{q} \right) = \epsilon \left( 1 + \frac{ \epsilon }{ p - \epsilon } \right)   + \frac{\epsilon^2 p}{q(p-\epsilon)} = \epsilon + \frac{\epsilon^2 (q+p)}{q(p-\epsilon)} .
\end{equation}
Combining the previous two lines and using $\epsilon^2 / ( q (p-\epsilon) ) > 0$ and $p,q <1$ yields \eqref{eqFiniteKLpqeps}.

We can now bound $\chi_2$ and $\chi_3$. For $\chi_2$, we choose $p = \mu_k, q = \mu_1, \epsilon = - \frac{\Delta_k}{C_n}$ in \eqref{eqFiniteKLpqeps} (note $- \mu_1 <  \frac{- \Delta_k }{ C_n} < 0 < \mu_k$, as required), and use $\mu_1 > \mu_k$ and $C_n > 2$, to obtain
\begin{equation} \label{eqChi2final}
\chi_2 \leq - \frac{\Delta_k}{C_n} - \frac{ \Delta_k \mu_k }{ C_n \mu_1 } + \frac{ 2 \Delta_k^2 }{  C_n^2\mu_1 ( \mu_k + \frac{\Delta_k}{C_n} ) } < - \frac{\Delta_k}{C_n} - \frac{ \Delta_k \mu_k }{ C_n \mu_1 } + \frac{ \Delta_k^2 }{ C_n \mu_k^2 }  .
\end{equation}
For $\chi_3$, we choose $p = 1 - \mu_k, q = 1 - \mu_1 , \epsilon = \frac{\Delta_k}{C_n}$ (note $\mu_1 - 1 < 0 < \frac{\Delta_k}{C_n} < 1 - \mu_k$) to obtain
\begin{align} \label{eqChi3final}
\chi_3 \leq \frac{\Delta_k}{C_n} + \frac{ \Delta_k ( 1 - \mu_k ) }{ C_n (1 - \mu_1) } + \frac{ 2 \Delta_k^2  }{ C_n^2 ( 1 - \mu_1 ) ( 1 - \mu_k - \frac{\Delta_k}{C_n} ) } < \frac{\Delta_k}{C_n} + \frac{ \Delta_k ( 1 - \mu_k ) }{ C_n (1 - \mu_1) } + \frac{ \Delta_k^2 }{ C_n  (1-\mu_1)^2 } . \ \ 
\end{align}

Finally, we can combine \eqref{eqChi1final}, \eqref{eqChi2final}, and \eqref{eqChi3final} to obtain
\begin{align}
\sum_{i=1}^3 \chi_i & \leq \frac{ \Delta_k^2 }{ C_n \mu_k (1-\mu_1) } + \frac{\Delta_k}{C_n} \left( \frac{1-\mu_k}{1-\mu_1} - \frac{\mu_k}{\mu_1} \right) + \frac{\Delta_k^2}{C_n} \left( \frac{1}{\mu_k^2} + \frac{1}{(1-\mu_1)^2} \right) \\
&  =  \frac{ \Delta_k^2 }{ C_n } \left(\frac{ 1 }{ \mu_k (1-\mu_1) } + \frac{1}{\mu_1(1-\mu_1)} +  \frac{1}{\mu_k^2} + \frac{1}{(1-\mu_1)^2} \right) < \frac{ \Delta_k^2 }{ C_n } \left(\frac{ 1 }{ \mu_k } + \frac{1}{1-\mu_1} \right)^2 , 
\end{align}
where the final inequality is $\mu_1 > \mu_k$. Hence, by \eqref{eqChiRewrite}, \eqref{eqFiniteKLmain2} holds.

\section{Proof of Proposition \ref{propAssumptions}} \label{secProofAssumptions}

For $g^{\texttt{UCB1}}$, Assumption \ref{assMono}, Assumption \ref{assConvSimple}, and the first part of Assumption \ref{assUpper} are immediate. For the second part of Assumption \ref{assUpper}, we use Pinsker's inequality (Proposition \ref{propKL}) to write
\begin{equation}
d ( x , g^{\texttt{UCB1}}_t ( x , s ) ) \geq 2 \left(  g^{\texttt{UCB1}}_t ( x , s ) - x \right)^2 = \log ( t^4 ) / s .
\end{equation}
Thus, we aim to show $t^4 \geq f(t)$. For $t=1$, this is an equality. For $t \in \{2,3,\ldots,\}$, we have
\begin{equation}
f(t) = 1 + t ( \log t )^2 < 1 + t^3 < 2 t^3 \leq t^4 ,
\end{equation}
as desired. The claimed expression for $S_t^{ g^{\texttt{UCB1}} } (x_1,x_2)$ and $S_t^{ g^{\texttt{UCB1}} } (x_1,x_2) / t \rightarrow 0$ in Assumption \ref{assConv1} are immediate. Finally, let $x_1, x_2 , \{ \delta_i \}_{i \in \N}$ be as in Assumption \ref{assConv2}. Then
\begin{equation} \label{eqStUCB1cont}
\frac{ S_t^{  g^{\texttt{UCB1}} } ( x_1 + \delta_i , x_2 - \delta_i ) }{ S_{ \ceil{(1-\delta_i)t}}^{  g^{\texttt{UCB1}} } ( x_1 - \delta_i , x_2 + \delta_i ) } < \frac{  \frac{2 \log t}{ (x_2 - x_1 - 2 \delta_i)^2 } +1   }{   \frac{2 \log ( (1-\delta_i) t )}{ (x_2 - x_1 + 2 \delta_i)^2   } } = \left( \frac{x_2 - x_1 + 2 \delta_i}{x_2 - x_1 - 2 \delta_i} \right)^2 \frac{1 + \frac{ (x_2 - x_1 - 2 \delta_i)^2 }{ 2 \log t } }{ 1 + \frac{ \log(1-\delta_i) }{ \log t } } ,
\end{equation}
where the inequality is by definition of $\ceil{ \cdot }$. Taking $t \rightarrow \infty$, then $i \rightarrow \infty$ on both sides, and deriving the analogous lower bound, establishes Assumption \ref{assConv2}.

For $g^{\texttt{KL-UCB}}$, Assumptions \ref{assMono} and \ref{assConvSimple} are immediate. For the remaining assumptions, first let $x \in (0,1)$ and $t,s \in \N$. Then since $d(x,x) = 0$, $d(x,1) = \infty$, and $d(x,\cdot)$ is strictly increasing and continuous on $(x,1)$, $g^{\texttt{KL-UCB}}_t(x,s)$ is the unique $y \in (x,1)$ s.t.\ $d(x,y) = \frac{ \log f(t) }{s}$. Assumption \ref{assUpper}, the expression for $S^{g^{ \texttt{KL-UCB} } }_t (x_1,x_2)$, and Assumption \ref{assConv1} follow. Finally, Assumption \ref{assConv2} follows from continuity of \eqref{eqKLdefn} in a manner analogous to \eqref{eqStUCB1cont}.

\section{Proof of Proposition \ref{propMessageAccuracy}} \label{secProofMessageAccuracy}

Fix $s \in \N$. Then for $s' \in \{1+\tau(\iota(s)-1) , \ldots, \tau(\iota(s)) \}$, $\tau ( \iota(s)-1) \leq s'-1 < \tau ( \iota(s) )$, so $\iota(s'-1) = \iota(s)-1$ by \eqref{eqDefnIota} and $\tau(\iota(s'-1)) = \eta(s)-1$ by definition. Therefore, we can rewrite the message $M_{k,s'}$ as
\begin{align}
M_{k,s'} %& = \Gamma_{ s' - \tau(\iota(s'-1))  } ( \hat{\mu}_{k, 1 +\tau(\iota(s'-1)  ) } )  - 2 \Gamma_{ s' - \tau(\iota(s'-1)) - 1 } ( \hat{\mu}_{k, 1 +\tau(\iota(s'-1)  ) }   ) \\
& = \Gamma_{ s' + 1 - \eta(s) } ( \hat{\mu}_{k,\eta(s) } ) - 2 \Gamma_{s' - \eta(s)} ( \hat{\mu}_{k,\eta(s) } )  .
\end{align}
We then rewrite the estimate $\bar{\mu}_{k,s}$ as
\begin{align}
\bar{\mu}_{k,s} %& = \sum_{ s' = 1 + \tau( \iota(s) - 1) }^{ \tau(\iota(s)) } 2^{ \tau( \iota(s) - 1) - s'  } M_{k,s'}   + 2^{ \tau(\iota(s)-1) - \tau(\iota(s)) } \\
& = \sum_{ s' = \eta(s) }^{ \eta(s)+\alpha(s)-1 } 2^{ \eta(s) - ( s' + 1 )  } \left( \Gamma_{ s' + 1 - \eta(s) } ( \hat{\mu}_{k,\eta(s) } ) - 2 \Gamma_{s' - \eta(s)} ( \hat{\mu}_{k,\eta(s) } ) \right) + 2^{ -\alpha(s) } \\
& = 2^{ -\alpha(s) } \Gamma_{ \alpha(s) } ( \hat{\mu}_{k,\eta(s) } ) - \Gamma_0 ( \hat{\mu}_{k,\eta(s) } )   + 2^{ -\alpha(s) } = 2^{ -\alpha(s) } \left( \Gamma_{ \alpha(s) } ( \hat{\mu}_{k,\eta(s) } ) + 1 \right) ,
\end{align}
where $\Gamma_0(x) = 0\ \forall\ x \in [0,1]$ by definition. The result follows by definition of $\Gamma$.

\section{Proof of Proposition \ref{propAlphaPsiMain}} \label{secProofAlphaPsi}

We begin with some intermediate results.
\begin{prop} \label{propTauRewrite}
Define $\tau : \Z_+ \rightarrow \Z_+$ as in \eqref{eqDefnTau}. Then
\begin{gather} 
\tau(i) = \sum_{j=1}^{\floor{\log_2 i}} 2^{j-1} j + ( i + 1 - 2^{\floor{\log_2 i} } ) ( \floor{\log_2 i}+1 )\ \forall\ i \geq 2, \quad \tau(1) = 1, \quad \tau(0) = 0 , \label{eqTauRewrite} \\
\tau(i+1) - \tau(i) = 1 + \floor{ \log_2 (i+1) }\ \forall\ i . \label{eqTauDiff} 
\end{gather}
\end{prop}
\begin{proof}
The $i \leq 1$ cases of \eqref{eqTauRewrite} are easily verified. For $i \geq 2$ and \eqref{eqTauDiff}, we first show
\begin{equation}\label{eqFloorLog}
\forall\ i \in \N,\ \floor{ \log_2 (i+1) } - \floor{ \log_2 i } = 1 (  \log_2 (i+1) \in \N ) .
\end{equation}
To prove \eqref{eqFloorLog}, first suppose $\log_2 (i+1) \in \N$. Then $2^{ \log_2 (i+1) - 1 } = 2^{ \log_2 (i+1) }  - 2^{ \log_2 (i+1) - 1 }  \leq 2^{ \log_2 (i+1) } - 1  = i  < i+1$, so $\log_2(i+1) - 1 \leq \log_2 i < \log_2 (i+1)$. These inequalities and $\log_2 (i+1) \in \N$ then imply $\floor{ \log_2 i } = \log_2 (i+1)-1 = \floor{ \log_2 (i+1) } - 1$, as claimed. Next, suppose $\log_2 (i+1) \notin \N$. Then $\floor{ \log_2 (i+1) } < \log_2 (i+1)$, so $2^{ \floor{ \log_2 (i+1) } } < i+1$. Since $2^{ \floor{ \log_2 (i+1) } } \in \N$, this implies $2^{ \floor{ \log_2 (i+1) } } \leq i = 2^{ \log_2 i }$, and thus $\floor{ \log_2(i+1) } \leq \log_2 i$. Together with $\log_2 i < \log_2 (i+1) < \floor{ \log_2(i+1) } + 1$, we conclude $\floor{ \log_2 i } = \floor{ \log_2(i+1)}$. 

From \eqref{eqFloorLog}, we derive another useful identity:
\begin{equation}\label{eqCeilFloorLog} 
\ceil{ \log_2 (i+1) } = 1 + \floor{ \log_2 i } .
\end{equation}
We prove \eqref{eqCeilFloorLog} by considering two cases. First, if $\log_2(i+1) \in \N$, then $1 + \floor{ \log_2 i } = \floor{ \log_2 (i+1) }  = \ceil{ \log_2 (i+1) }$, where we used \eqref{eqFloorLog} and $\log_2(i+1) \in \N$. Next, if $\log_2 (i+1) \notin \N$, \eqref{eqFloorLog} implies $1 + \floor{ \log_2 i } = 1 + \floor{ \log_2 (i+1) } = \ceil{ \log_2 (i+1) }$.

We next note the following identity, which is easily proven by induction on $m$:
\begin{equation}\label{eqSuperGeom}
\sum_{j=1}^m 2^{j-1} j = 1 + 2^m ( m - 1 )\ \forall\ m \in \N .
\end{equation}

Combining \eqref{eqSuperGeom}, \eqref{eqCeilFloorLog}, and definition of $\tau$, we obtain \eqref{eqTauRewrite}:
\begin{align}
& \sum_{j=1}^{\floor{\log_2 i}} 2^{j-1} j + ( i + 1 - 2^{\floor{\log_2 i} } ) ( \floor{\log_2 i}+1 )  \\
& \quad\quad = 1 + 2^{ \floor{ \log_2 i } } ( \floor{\log_2 i} - 1 ) + ( i + 1 - 2^{\floor{\log_2 i} } ) ( \floor{\log_2 i}+1 ) \\
& \quad\quad = 1 + (i+1) ( \floor{\log_2 i}+1 ) -2^{\floor{\log_2 i} + 1}  =  1 + (i+1) \ceil{ \log_2(i+1) } - 2^{ \ceil{ \log_2 (i+1) } } = \tau(i) .
\end{align} 

Finally, we prove \eqref{eqTauDiff} by considering two cases. First, if $\log_2 (i+1) \notin \N$, \eqref{eqFloorLog} implies $\floor{ \log_2 (i+1) } = \floor{ \log_2 i }$, so by \eqref{eqTauRewrite}, $\tau(i+1) - \tau(i) = \floor{ \log_2 (i+1) } + 1$. Next, if $\log_2 (i+1) \in \N$, then $\floor{ \log_2 (i+1) } = \floor{ \log_2 i } + 1$ by \eqref{eqFloorLog}, so by \eqref{eqTauRewrite},
\begin{align}
\tau(i+1) - \tau(i) & = 2^{ \floor{ \log_2 (i+1) } - 1 } \floor{ \log_2 (i+1) } + ( \floor{\log_2 (i+1)} + 1 ) ( i +2 - 2^{ \floor{ \log_2 (i+1) } } ) \\
& \quad\quad - \floor{\log_2 (i+1)}  ( i+1 - 2^{ \floor{ \log_2 (i+1) } - 1 } )  = \floor{\log_2(i+1)} + 1 . \qedhere
\end{align}
\end{proof}

\begin{prop} \label{propPsiUpper}
For any $s \in \N$, $\eta(s) \leq 2^{\alpha(s)}\alpha(s)$.
\end{prop}
\begin{proof}
First note that by definition of $\alpha$ and Proposition \ref{propTauRewrite},
\begin{equation}\label{eqAlphaIota}
\alpha(s) =  \tau ( \iota(s) ) - \tau ( \iota(s) - 1 ) = 1 + \floor{ \log_2 \iota(s) } ,
\end{equation}
which clearly implies
\begin{equation}\label{eqPsiUpperA}
2^{\alpha(s)} \alpha(s) = 2^{ 1 + \floor{ \log_2 \iota(s) } } ( 1 + \floor{ \log_2 \iota(s) } ) .
\end{equation}
On the other hand, by definition of $\eta$, $\tau$, and $\iota$, 
\begin{equation}\label{eqPsiUpperB}
\eta(s) = 1 + \tau ( \iota(s)-1) = \iota(s) \ceil{ \log_2 \iota(s) }  + 2 - 2^{ \ceil{ \log_2 \iota(s) } } .
\end{equation}
Now consider two cases. First, if $s$ is s.t.\ $\iota(s) = 1$, then \eqref{eqPsiUpperA} yields $2^{\alpha(s)} \alpha(s) = 2$, while \eqref{eqPsiUpperB} yields $\eta(s) = 1$. Next, if $s$ is s.t.\ $\iota(s) \geq 2$, then $2 - 2^{ \ceil{ \log_2 \iota(s) } } \leq 0$, so by \eqref{eqPsiUpperA} and \eqref{eqPsiUpperB},
\begin{align}
2^{ \alpha(s) } \alpha(s) \geq \iota(s) \ceil{ \log_2 \iota(s) }\geq \iota(s) \ceil{ \log_2 \iota(s) } + 2 - 2^{ \ceil{ \log_2 \iota(s) } } & = \eta(s) . \qedhere
\end{align}
\end{proof}

\begin{prop} \label{propPsiLower}
For any $s \in \N$, $\eta(s) \geq s - 2 \log_2 ( s) -1$.
\end{prop}
\begin{proof}
First note that by definition of $\eta$ and \eqref{eqAlphaIota},
\begin{align}
\eta(s) & =\tau ( \iota (s) + 1 ) + 1 - \left( \tau ( \iota (s) + 1 ) - \tau ( \iota (s)  ) + \tau ( \iota (s)  ) - \tau ( \iota (s) - 1 ) \right) \\
& = \tau ( \iota (s) + 1 ) - 1 -  ( \floor{ \log_2 ( \iota(s)+1) } + \floor{ \log_2 \iota(s) } ) .
\end{align}
Next, recall that by definition of $\tau$ and $\iota$, $\tau ( \iota (s) + 1 ) > s$, which implies (since $\tau : \Z_+ \rightarrow \Z_+$) that $\tau ( \iota (s) + 1 ) \geq s+1$. Furthermore, by \eqref{eqFloorLog}, $\floor{ \log_2 ( \iota(s)+1) } + \floor{ \log_2 \iota(s) } \leq 1 + 2 \log_2 \iota(s)$. Combining the above, we obtain $\eta(s) \geq s - 2 \log_2 ( \iota(s) ) - 1$. Thus, to complete the proof, we aim to show $\iota(s) \leq s$. Toward this end, note that since $s \geq \tau(\iota(s))$ by definition and $\tau$ is increasing by Proposition \ref{propTauRewrite}, it suffices to prove $\tau(i) \geq i\ \forall\ i \in \Z_+$. This follows from Proposition \ref{propTauRewrite}: it holds with equality for $i \in \{0,1\}$, and, for $i \in \{2,3,\ldots\}$,
\begin{align}
\tau(i) \geq \sum_{j=1}^{ \floor{ \log_2 i } } 2^{j-1} + i +1 - 2^{ \floor{ \log_2 i } } & = i . \qedhere
\end{align}
\end{proof}

\begin{prop} \label{propLogPolyIneq}
Define $h_1, h_2, h_3 : \R_+ \rightarrow \R$ by
\begin{equation}
h_1(x) = x - 3 \log_2 x , \quad h_2(x) = \sqrt{x} -  \log_2 x , \quad h_3(x) = x - 4 \log_2 x -2  \quad \forall\ x \in \R_+ .
\end{equation}
Then $h_1(x) \geq 0\ \forall\ x \geq 16$, $h_2(x) \geq 0\ \forall\ x \geq 16$, and $h_3(x) \geq 0\ \forall\ x \geq 22$.
\end{prop}
\begin{proof}
We first compute $h_1(16) = 4$, $h_2(16) = 0$, and $h_3(22) \geq 22 - 4 \log_2 32 -2 = 0$. Next, we compute the corresponding derivatives:
\begin{equation}
h_1'(x) = 1 - \frac{3}{x \log 2} , \quad h_2'(x) = \frac{1}{\sqrt{x}} \left( \frac{1}{2} - \frac{1}{\sqrt{x} \log 2} \right) , \quad h_3'(x) = 1 - \frac{4}{x \log 2} .
\end{equation}
Note  $\log 2 \approx 0.69 > 0.5$, so $h_1'(x) , h_2'(x) \geq 0\ \forall\ x \geq 16$ and $h_3'(x) \geq 0\ \forall\ x \geq 22$. Thus, $h_1(16) \geq 0$ and $h_1$ increases on $[16,\infty)$, which implies $h_1(x) \geq 0\ \forall\ x \geq 16$. The $h_2$ and $h_3$ inequalities are argued similarly.
\end{proof}

We now prove Proposition \ref{propAlphaPsiMain}. Since $\iota$ is increasing by definition, $\tau$ is increasing by Proposition \ref{propTauRewrite}, and $\eta(s) = 1 + \tau(\iota(s)-1)$, it suffices to show $\eta ( \floor{ u + 3 \log_2 u }  )  \geq u$. We have
\begin{align}
\eta ( \floor{ u + 3 \log_2 u }  ) & \geq  \floor{ u + 3 \log_2 u }  - 2 \log_2 ( \floor{ u + 3 \log_2 u } ) - 1 \\
& \geq u + 3 \log_2 u - 2 \log_2 ( u + 3 \log_2 u ) - 2 \\
& = u + \log_2 \left( \frac{u^3}{4 ( u + 3 \log_2 u )^2 } \right)   \geq u + \log_2 \left( 4 \left( \frac{u}{u +3 \log_2 u } \right)^2 \right) \geq u ,
\end{align}
where the first inequality holds by Proposition \ref{propPsiLower}, the second by $x-1 \leq \floor{x} \leq x$, and the third and fourth by $u \geq 16$ (so $u - 3 \log_2 u \in (0,u)$ by the first inequality in Proposition \ref{propLogPolyIneq}).

Next, we observe $\eta(s) \geq s/2$ follows from Proposition \ref{propPsiLower} and the second inequality in Proposition \ref{propLogPolyIneq} (note $u \geq 16$ implies $s \geq u + 3 \log_2 u \geq 28$, so this inequality applies):
\begin{equation}
\eta(s) \geq s - 2 \log_2 (s) - 1 = \frac{s}{2} + \frac{1}{2} ( s - 4 \log_2(s) - 2 ) \geq \frac{s}{2} .
\end{equation}

Finally, to show $\alpha(s) \geq \log_2 \sqrt{s}$, we assume instead that $\alpha(s) <\log_2 \sqrt{s}$. Then $\eta(s) < \sqrt{s} \log_2 \sqrt{s}$ by Proposition \ref{propPsiUpper}. Combined with $\eta(s) \geq s/2$, this implies $\sqrt{s} < \log_2 s$, contradicting the third inequality in Proposition \ref{propLogPolyIneq} (which holds since $s \geq 16$).

\section{Existing results}

\begin{prop} \label{propKL}
$\forall\ 0 < p < q < r < 1$, $d(p,r) \geq d(p,q) + d(q,r)$ and $d(p,q) \geq 2 ( p-q)^2$.
\end{prop}
\begin{proof}
See, e.g., \cite[Lemma 10.2]{lattimore2020bandit}.
\end{proof}
 
\begin{prop} \label{propHoeffding}
Let $Z = \frac{1}{m} \sum_{i=1}^m Z_i$, where $m \in \N$ and $\{ Z_i \}_{i=1}^m$ are i.i.d.\ $[0,1]$-valued random variables. Then for any $t > 0$, $\P ( Z \geq \E Z + t ) , \P ( Z \leq \E Z - t ) \leq e^{-2 t^2 m}$.
\end{prop}
\begin{proof}
See, e.g., \cite[Theorem 1]{dubhashi2009concentration}.
\end{proof}

\end{document}